\documentclass[runningheads]{llncs}
\usepackage[T1]{fontenc}
\usepackage{graphicx}
\usepackage{booktabs}
\usepackage[misc]{ifsym}
\newcommand{\corr}{(\Letter)}
\usepackage{graphicx}
\usepackage{amsmath}
\usepackage{amssymb}
\usepackage{subcaption}
\usepackage{booktabs}
\usepackage{algorithm}
\usepackage[hidelinks]{hyperref}

\usepackage{mwe}
\DeclareMathOperator*{\argmax}{arg\,max}
\newcommand{\citet}[1]{\citeauthor{#1}~[\citeyear{#1}]}

\begin{document}
%todo: change title
\title{Measuring Explainer Stability via Attribution Separability}

% Measuring Ranking Stability in Attribution Methods

% Analyzing Feature Importance Distribution to Measure AM stability

% A Strategy to Measure Attribution Stability via Distributional Analysis

% Linking Attribution Stability and Separability 

% Measuring Explainer Stability via Attribution Separability

%falta XAI, Explainability 

\titlerunning{Measuring Explainer Stability via Attribution Separability}
% If the full title of your paper is short enough to also fit in the running head, you can omit the abbreviated paper title here. You can check as follows: if you comment out the \titlerunning line, something will appear in the header of all odd-numbered pages of your PDF from page 3 onward. This something is either the full title (in which case all is well), or the error message "Title Suppressed Due to Excessive Length". If this error message appears, you're going to want to provide an abbreviated title within the \titlerunning command, because if you won't do it, Springer will do it for you.

%N.B.: Author information (both in the \author{} and \authorrunning{} command) should only be present in the Camera-Ready Version of your paper. The version that you initially submit for review, ought to be double-blind. So, when initially submitting your paper, use:
\author{Eddie Conti$^1$ \corr, 	
Álvaro Parafita$^1$, Axel Brando$^1$}
%\author{Andr\'e Lauren Benjamin\inst{1} \and
%Calvin Cordozar Broadus Jr.\inst{2,3} \corr \and
%Antwan Andr\'e Patton\inst{1}\orcidID{0000-1111-2222-%3333}}
% You may leave out the orcidID information, if you want to.
% Use \corr to indicate the corresponding author. Note the spacing around the \corr command. Only one author can be the corresponding author.

%N.B.: comment out the \authorrunning{} command for the double-blind version of your paper submitted for review. Later, if your paper is accepted, use the command for the Camera-Ready Version.
\authorrunning{E. Conti et al.}
% If there are two authors, write 'A.L. Benjamin and C.C. Broadus Jr.'
% If there are more than two authors, '[...] et al.' is used.
\institute{Barcelona Supercomputing Center, Barcelona, Spain \email{econti@bsc.es}
\email{parafita.alvaro@gmail.com}
\email{axelbrando@gmail.com}}
%\author{Andr\'e Lauren Benjamin\inst{1} \and
%Calvin Cordozar Broadus Jr.\inst{2,3} \corr \and
%Antwan Andr\'e Patton\inst{1}\orcidID{0000-1111-2222-%3333}}
% You may leave out the orcidID information, if you want to.
% Use \corr to indicate the corresponding author. Note the spacing around the \corr command. Only one author can be the corresponding author.

%N.B.: comment out the \authorrunning{} command for the double-blind version of your paper submitted for review. Later, if your paper is accepted, use the command for the Camera-Ready Version.
%\authorrunning{A.L. Benjamin et al.}
% If there are two authors, write 'A.L. Benjamin and C.C. Broadus Jr.'
% If there are more than two authors, '[...] et al.' is used.

%\institute{Fictional Southern University, Savannah GA 31404, USA \email{\{a.l.benjamin,a.a.patton\}@fsu.fake}
%\and
%Fictional West Coast University, Long Beach CA 90840, USA \email{ccb@fwcu.fake}
%\and
%Secondary European Affiliation, Tiergartenstr. 17, 69121 Heidelberg, Germany
%\email{lncs@springer.com}}

\maketitle              

\begin{abstract}
Attribution methods (AMs) assign an importance score to each feature and are widely adopted to explain black-box models. However, most methods can produce variable attribution scores due to stochastic components in their definition. In this paper, we propose a distribution-based framework to capture the stability of attribution scores. In particular, our approach allows to understand the degree of separability in the ranked attribution vector and obtain the largest index for which a feature ranking remains reliable. We further extend this framework to compare AMs based on the robustness of their rankings across a dataset. Through experiments, we demonstrate how to apply our method to evaluate explainer stability. Overall, our approach provides a complementary criterion for evaluating the stability of AMs.

\end{abstract}

\section{Introduction} \label{sec:intro}
Explainable AI (XAI) aims to make the decision-making of models understandable to humans. In recent years, the adoption of complex architectures in high risk scenarios has raised concerns among researchers for its potential ethical and social implications, such as unfair treatment, bias amplification or discrimination \cite{fairness}, \cite{ethics_ai}. For these reasons, it is essential to understand how these models reason and generate predictions.

AMs or explainers \cite{intrinsic_posthoc} aim to detect which features are most relevant in generating the model output (feature attribution). In formal terms, we consider a model that maps input data to an output variable, $f: \mathcal{X} \to Y$, where $f(x)$ is the prediction for a given instance $x\in \mathcal{X} \subseteq \mathbb{R}^n$.
In general, $Y$ may represent either discrete classes or a continuous range of values;
in this work, we focus on the common binary classification setting, where $Y = \{0, 1\}$. Then, an AM is a function $\alpha_f\colon \mathbb{R}^n \to \mathbb{R}^n$, denoted $\alpha$ for simplicity, 
where for $\alpha(x)=(\alpha_1,\ldots,\alpha_n)$, $\alpha_i$ represents the importance of feature $x_i$ for the model’s decision process. In the literature we refer to them as local explanations, because they analyze a specific instance of the dataset. 

% \begin{figure}[t]
%     \centering
%     \includegraphics[width=\columnwidth]{figures/Framework.pdf}
%     \caption{Graphical illustration of our framework. Given an explainer $\alpha$ with a stochastic component, multiple runs yield different attribution values for each feature (here $\alpha_1,\alpha_2$) forming a distribution. The overlap between these distributions provides a measure of stability.}
%     \label{fig:framework}
% \end{figure}

Over the past decade, the XAI community has developed a wide range of AMs \cite{Molnar,methods_overview}, alongside several proposals for their evaluation \cite{kadir2023evaluation,agarwal2022openxai}. Despite this methodological richness, explanations suffer from various critical issues ranging from variability, inconsistency, and untruthfulness \cite{agree/disagree,ju2022logic}.

In this work we address a fundamental challenge in XAI: many AMs incorporate \textbf{stochastic components}—such as Monte Carlo sampling in SHAP \cite{SHAP} and LIME \cite{LIME}, or random perturbations in DiCE \cite{dice}—which can \textbf{yield explanations that vary significantly across runs}. This instability, acknowledged by several studies \cite{first_stability,Slack2019FoolingLA,dombrowski2019explanations}, raises concerns about reproducibility and reliability. As highlighted by Pawlicki \cite{pawlicki2023towards}, stability is a prerequisite for user trust and for the generalizability of insights derived from explanations.

In this work, we focus on explainers with stochastic components and introduce a novel distribution-based strategy to \textbf{quantify the stability of feature rankings}. In particular, we propose a metric that captures the maximum $k$ for which the top-$k$ ordering remains robust across multiple executions of an AM. From this proposal, we can measure a specific aspect of explanation stability, namely the degree of separability between feature importance values in the attribution vector. In particular, our contributions are:
\begin{itemize}
    \item We introduce a novel approach to quantify the stability of AM rankings based on distributional analysis. Our approach, mathematically grounded, allows to assess to what extent the ranking is significant.  
    \item We formally show---and validate empirically---how feature importance separation influences the overall stability of AMs. 
    \item We conducted experiments demonstrating the application of our metric to analyze and compare AMs across different models and datasets.

\end{itemize}

\section{Related Work}
The issue of explanation stability has been investigated deeply in the literature. A first set of approaches aligns with the taxonomy identified by Carvalho et al. \cite{carvalho2019machine} and Gawantka et al. \cite{Gawantka2024}, which define stability respectively as ``how similar the explanations are for similar instances'' (an idea also supported by Alvarez-Melis \cite{alvarez2018towards}) or ``high stability of an explanation is observed when the explanation undergoes minimal changes in response to minor variations in unimportant features of the data instance.'' Consistent with this definition, metrics have been developed that model perturbations at the input level. Alvarez-Melis et al. \cite{first_stability} formalize the first stability metric for local explanation methods, arguing that explanations should be robust to local perturbations of the input. 
% Accordingly, it is defined as:
% \begin{equation} \label{eq:first_stability}
%     S(x,\alpha) = \max_{x'\in U(x)} \frac{||\alpha(x)-\alpha(x')||}{||x-x'||},
% \end{equation}
% where $x'$ lies in a proper neighborhood $U(x)$ of $x$ and the model yields the same prediction (e.g. the same predicted class in the case of a classification problem) for both $x$ and $x'$.
Agarwal et al. \cite{rethinking_stability} deem that the approach of Alvarez-Melis et al. \cite{first_stability} does not leverage potentially meaningful information---such as the model’s internal representations---for evaluating stability, and implicitly assumes that $f$ behaves similarly on inputs $x$ and $x'$ that are close. As a consequence, authors propose several \textit{relative stability} metrics: Relative Input Stability, which measures the relative distance between explanations w.r.t the distance between inputs; Relative Representation Stability, which instead uses internal representations of $x$ and $x'$; and Relative Output Stability, which computes the relative distance replacing $x$ and $x'$ with their logits outputs. 
Similarly, Butt et al. \cite{stability_perturbation} focus on consistency in feature importance values across perturbed inputs of a given instance 
$x$: they generate a set of perturbations and aggregate the resulting variations in feature importance across these perturbations.

From a different perspective but still at input-level, Pawlicki \cite{pawlicki2023towards} study the stability of SHAP under three types of input perturbations: randomly shuffling values, randomly inserting the median, and adding Gaussian noise, to investigate how attribution values are affected. The author concludes that shuffling features has a more significant impact on SHAP stability, and that the stability of explanations varies across datasets, likely due to differences in complexity and characteristics---as also observed by Butt et al. \cite{stability_perturbation}.

A different approach \cite{mult_smoothing} proposes the concept of explanation stability accounting for perturbations directly on $\alpha$. Summarizing their idea, they consider $\alpha_i \in \{0,1\}$ (i.e., features are either relevant or not relevant), and stability means that the prediction does not change even if more explanatory features are added to $\alpha(x)$. To account for small perturbations, the authors consider to alter few entries of $\alpha$.
Hence the robustness of explanations is studied by analyzing whether small modifications in feature selection affect the model’s prediction. 

In this scenario, to the best of our knowledge, we introduce a novel approach that does not rely on perturbations or modifications of the explainer. Instead, we analyze the stability of an AM with a \textbf{distributional analysis}. Specifically, we estimate the distribution of each feature’s importance scores across runs and employ a \textbf{metric to quantify the overlap between these distributions}. A lower overlap indicates that the feature rankings are more distinct and, consequently, more reliable. This allows us to assess whether the resulting ranking is significant and to what extent, while also enabling comparison across AMs.

\section{Problem Definition}\label{sec:framework}
We now formalize our notion of stability. Specifically, we seek to answer the following:
\begin{center}
\textit{
\textbf{Research Question}: given an ordered vector of feature importances produced by an AM, to what extent can we trust the resulting ranking? More specifically, how confident can we be that the first feature is truly more important than the second, the second more important than the third, and so on?
}
\end{center}

To address this issue, we model each feature attribution as a distribution rather than a single deterministic value. Given the attribution vector $\alpha(x) = (\alpha_1(x), \dots, \alpha_n(x))$ and its associated dimension-wise distributions $p_1,\ldots,p_n$, we define the induced ranking $r(x)$ as the permutation of features sorted according to their expected attribution values:

\begin{equation}
\label{eq:mean_vector}
\tilde{\alpha}_i = \mathbb{E}_{\alpha_i \sim p_i}[\alpha_i]
\quad\Rightarrow\quad
r(x)=\mathrm{argsort}(\tilde{\alpha}(x)).
\end{equation}

The key question then becomes whether the ordering induced by $r(x)$ is reliable. In particular, we seek to quantify how distinguishable consecutive features in the ranking are and, consequently, how much confidence can be placed in their relative positions. For this purpose, we perform pairwise comparisons between the distributions $p_i$ and $p_j$ associated with each pair of features in $r(x)$. For every pair, we compute a distance $d(p_i,p_j)$, where $d$ is a metric taking values in a bounded interval (e.g., $[0,1]$) and measuring the degree of separation between the two distributions.

Our notion of attribution separability provides a flexible way to analyze a key component of stability from multiple perspectives. Within this general framework, we focus on two main aspects: (i) the \textbf{evolution of separability} across the attribution vector, and (ii) the notion of \textbf{$\boldsymbol{k}$-stability}, which identifies the largest prefix of top-ranked features whose ordering can be considered reliable. In Appendix \ref{sec:appenix_GB}, we show how the framework naturally \textbf{extends to binary explanations or group-based explanations} (i.e., $\alpha \in [0,1]^n$). 

Finally, we emphasize that the framework uses the \textbf{assumption} of local independence, that is, the features used by a model can be considered independent within a neighborhood around the sample to explain. This assumption is common in the AM literature: see \cite{SHAP,LIME,vstrumbelj2014explaining,samiei2021addressing,goldwasser2024statistical,leemann2023post}.\\

\subsection{The choice of distance metric}

Although the framework can use any metric, in this paper we use the following measure to quantify the distance between two distributions, introduced by Conti et al. \cite{conti2026cid}: 
\begin{equation} \label{eq:distance}
d(p,q) = 1 - \frac{\int_{\text{supp}(p) \cap \text{supp}(q)} \min(p(x),q(x)) dx}{\int_{\text{supp}(p) \cup \text{supp}(q)} \max(p(x),q(x)) dx}.
\end{equation}
This measure is the continuous generalization of the Jaccard distance, a metric widely adopted to measure the overlap between sets. We employ this $d$ since it is proved to be a metric, with values in $[0,1]$, and particularly suitable to capture distributional discrepancies.

To validate our framework we prove the following.
\begin{proposition}[Well-posedness in the ideal case] \label{prop:well-posedness}
Let $x\in\mathbb{R}^n$ be a point to be explained with ground-truth importances 
$f_1>\cdots>f_n$. Consider an AM $\alpha$ with a stochastic component
that produces, at each run $s=1,\ldots,m$, an attribution vector 
$\alpha^{(s)}=\{\alpha_1^{(s)},\ldots,\alpha_n^{(s)}\}$. 
Denote by $p_{i,m}$ the empirical probability density function associated with the values 
$\{\alpha_i^{(s)}\}_{s=1}^m$.  

Assume that, for each $i$, $p_{i,m}$ converges in distribution to the Dirac measure 
$\delta_{f_i}$ as $m \to \infty$. Then, for the distance metric $d(\cdot,\cdot)$ defined
in \eqref{eq:distance}, we have
\[
d(p_{i,m},p_{i+1,m}) \to 1 \quad \text{as } m \to \infty,
\quad \forall i \in \{1,\ldots,n-1\}.
\]
\end{proposition}
This result, proved in Appendix \ref{sec:theoretical_appendix}, confirms that the metric is well-posed in the sense that, if an AM were perfectly faithful and its stochasticity vanished asymptotically, then it would achieve the maximum possible stability.

\subsection{$k$-Stability}

While the evolution of separability can be directly analyzed by tracking how the values of $d(p_i, p_j)$ change across the attribution vector, the notion of $k$-stability requires a formal definition. In many applications, an AM is used to explain a prediction by highlighting the most relevant features; thus, we are particularly interested in assessing how reliably the top portion of the ranking can be trusted. This motivates the following definition:
\begin{definition} \label{def:kstable}
Given a threshold $l \in [0,1]$ and an instance $x$, we say that the attribution method $\alpha(x)$ is \emph{$k$-stable}, where
\[
k := \argmax_{s \in [n]}\,\, \text{such that } d(p_i, p_j) \ge l 
\quad \forall\, i,j \in \mathcal{A}_s(r(x)),
\]
and $\mathcal{A}_s(r(x))$ denotes the set of the top $s$ indices in the ranked list $r(x)$.
\end{definition}
In simple terms, $k$ represents the largest prefix of the ranking $r(x)$ for which all consecutively-ranked pairs $(p_i, p_j)$ have a distance $d(p_i, p_j)$ exceeding the threshold $l$. This formulation allows us to assess the degree of separability among the top-ranked features in the attribution ranking. While it is possible that features lower in the ranking (i.e., the "tail") may also be well-separated in terms of $d$, the focus is to ensure that the top positions are genuinely distinguishable because we expect those features to be more relevant.

% To illustrate the computation of the separability of attribution scores, in \autoref{fig:kde_plot} we compute the value of LIME for the first entry of the test set of the Diabetes Dataset (\cite{diabetes_dataset}). We run the LIME explainer $50$ times for that instance and collect the attribution scores for each feature. If we now follow \autoref{def:kstable} for $\tilde{\alpha}_i$, we can say that at level $0.9$ LIME is \textit{5-stable}. Indeed, as it can be seen in the plot, the tail variables Insulin, SkinThickness and BloodPressure share a high overlap with $d$-values $0.23$ and $0.19$ for their ordered pairwise comparison. We conclude that for this instance, the first $5$ feature importance scores are stable w.r.t. the LIME explainer, in the sense that employing the $d$ metric these features are separable with level $0.9$, while the last 3---as it can be seen in the plot---are almost undistinguishable.

% \begin{figure}[t]
%     \centering
%     \includegraphics[scale=0.2]{figures/kde_all_together.pdf}
%     \caption{LIME values for the features of the first test instance of Diabetes dataset. The explainer was run $50$ times and we employed Gaussian Estimation for the distributions. It can be observed that the rightmost five pairs of distributions are effectively separate, while the leftmost three graphs are practically indistinguishable.}
%     \label{fig:kde_plot}
% \end{figure}

The notion of \textit{k-stability} induces a total order over the set of AMs for the instance $x$, $\mathbb{A}_{x} : = \{\alpha(x) : \text{$\alpha$ is an AM} \}$.
\begin{definition} \label{def:comparison}
Let $\alpha_1(x), \alpha_2(x) \in \mathbb{A}_{x}$ be two attribution maps for the same instance $x$, with corresponding stability levels $k_1$ and $k_2$. Given a fixed threshold $l \in [0,1]$, we say that $\alpha_1(x)$ is more stable that $\alpha_2(x)$ if $k_1\geq k_2$.
\end{definition}
Combining \autoref{def:kstable} and \autoref{def:comparison}, two AMs are compared by the proportion of instances in $\mathcal{X}$ for which one exhibits higher k-stability than the other; the AM that prevails on the majority of instances is considered more stable on average.

\section{Validation of the proposed stability metric} \label{sec:validation}
Before turning to the experiments, we need to clarify that the adopted strategy is directly aligned with the research question we aim to address. In particular, we want to show that the proposed metric---and therefore the resulting k-stability values---are sensitive to the degree of separability between feature importance values, which in turn reflects the reliability of the induced ranking.

To support this intuition, for this demonstration, in Appendix \ref{sec:theoretical_appendix} we consider a simplified setting where attribution scores are assumed to follow Gaussian distributions centered at their ground-truth values. Under this assumption, the distance between two distributions can be expressed as:
\[
d(p,q)
= 1 - \frac{2(1 - \Phi(z))}{2\Phi(z)}
= \frac{2\Phi(z) - 1}{\Phi(z)},
\]
where $\Phi$ is the cumulative distribution function of the standard Gaussian and $z = \frac{\Delta}{2\sigma_i}$, with $\Delta$ denoting the difference between the ground-truth attribution values of two features and $\sigma_i$ their shared standard deviation. From this expression, it follows that when $\Delta \to 0$, we obtain $d(p,q) \to 0$, while in the limit $\Delta \to \infty$, we have $d(p,q) \to 1$. 

This shows that the proposed distance is intrinsically sensitive to the separation between feature importances: the more distinguishable the ground-truth attributions are, the higher the induced distance. This supports the interpretation that attribution separability is a fundamental component of stability, and justifies its role in the definition of k-stability.

This theoretical result is supported by empirical evidence. As detailed in the supplementary material (\ref{sec:sythetic_empirical_validation}), we construct synthetic datasets with known ground-truth feature importances to evaluate the behavior of the metric in a controlled setting. \autoref{fig:dataset_complexity} shows that, when keeping the range of importance values fixed ($[0,0.4]$) and increasing the number of features, feature importances naturally begin to overlap, increasing their stability, with our metric capturing this phenomenon: we observe a clear downward trend in $k$-stability, indicating that the metric is directly affected by the reduced separability.

\begin{figure}[!h]
     \centering
     \includegraphics[scale=0.22]{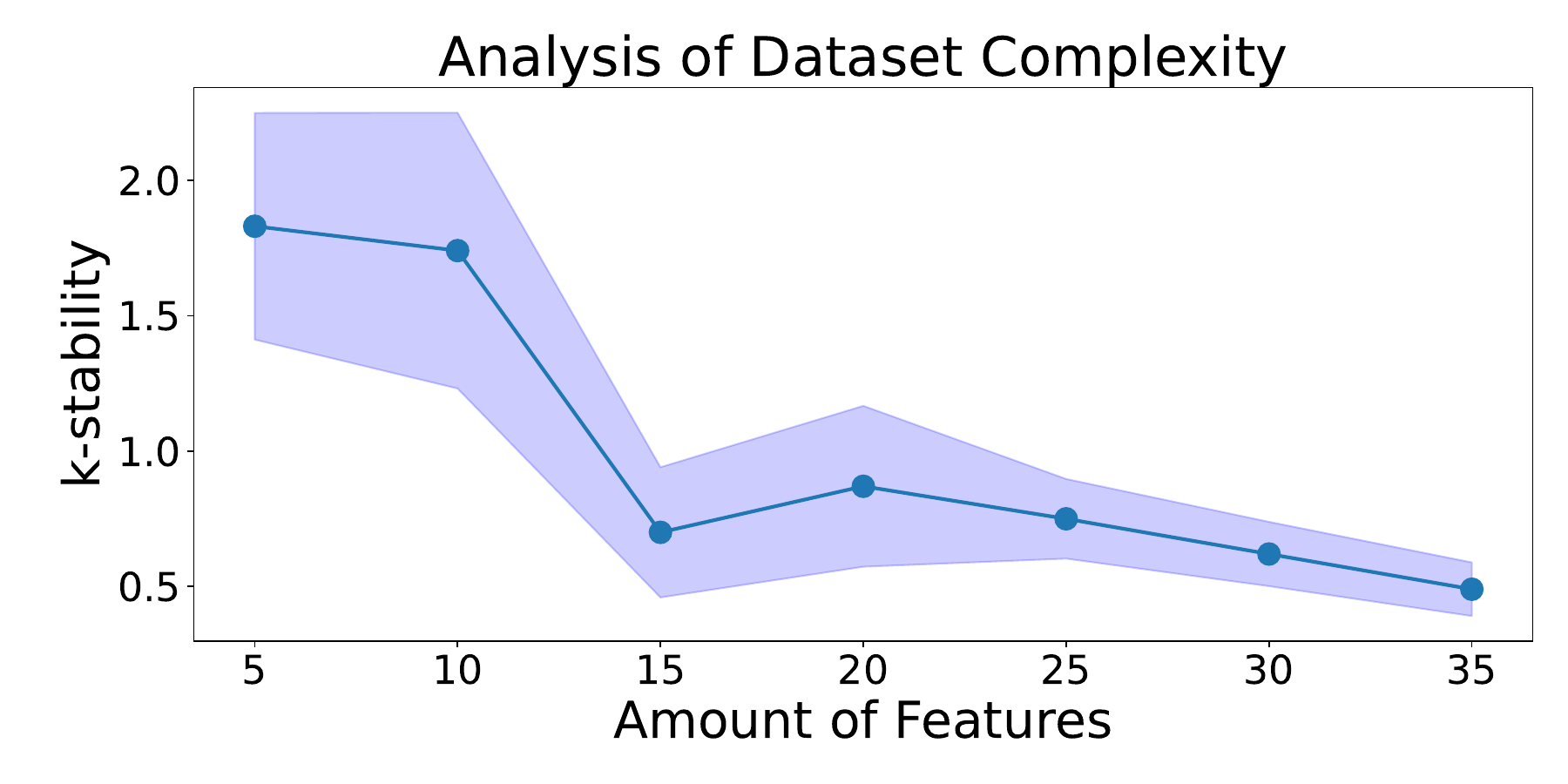}
     \caption{When feature importance values fall in a certain range (here $[0,0.4]$), and we increase the number of features, the k-stability value naturally decreases.}
     \label{fig:dataset_complexity}
 \end{figure}

To further confirm this behavior, we also decrease separability by sampling ground-truth importances from progressively narrower intervals ($[0.4,1]$, $[0.6,1]$, and $[0.8,1]$) for $5$ and $10$ features. As expected, smaller intervals lead to lower separability, which in turn decreases $k$-stability. See the supplementary material for these results. Overall, these controlled experiments provide evidence that the proposed metric responds coherently to changes in feature importance separability, supporting its adequacy for the research question.

% Taken together, these findings highlight attribution separability as one concrete source of instability, complementing existing studies \cite{agree/disagree,pawlicki2023towards,ribeiro2021does,dai2022fairness,velmurugan2021evaluating} that attribute instability to factors such as model complexity, dataset structure, or explainer stochasticity. Our results, aligned with this literature, suggest that instability does not originate from a single cause but rather from multiple, interacting components. This perspective underscores the need to interpret \textbf{instability as a multidimensional phenomenon} shaped by the dataset, the model and the choice of AM.

% Furthermore, this perspective resonates with established taxonomies in uncertainty quantification \cite{hullermeier2021aleatoric}, which distinguish between aleatoric (irreducible noise), epistemic (model- or data-based), and domain (e.g., out-of-distribution or adversarial) uncertainty. Drawing this parallel highlights the inherent connections between both fields, for which reason we advocate for future work to focus on disentangling the different sources of instability and exploring its connections to Uncertainty Quantification as a promising research direction.

% Change to Experiment y separar en parrafos model, parameters, etc... y hacer experiment list. Other experiments con todo los iperparametros y analisis
\section{Experiments} \label{sec:experiments}
In this section, we empirically evaluate the proposed framework introduced in Section \ref{sec:framework} and demonstrate the applicability of the stability metric for analyzing AMs. Experimental code will be made publicly available upon acceptance.

\textbf{Dataset.} We use four benchmark datasets: Diabetes \cite{diabetes_dataset}, Heart Disease \cite{heart_disease}, Mobile \cite{Mobile_price}, and Churn \cite{churn_dataset}.

\textbf{Methods.} We consider four explanation methods: \textbf{LIME}, \textbf{SHAP}, \textbf{DiCE}, and a random baseline (\textbf{RAND}) used for calibration. Given the connection between LIME and SHAP \cite{Molnar}, we ensure consistency by using the same background dataset and the same \verb|num_samples| parameter. Across all experiments, we evaluate 50 instances from each test set and run each explainer 50 times per data point. 

\textbf{Stability Estimation.} Our approach relies on estimating the distributions $p_1, \ldots, p_n$ of attribution scores. In practice, we approximate each distribution from the previous 50 explainer runs using Gaussian Kernel Density Estimation (KDE) with Silverman’s rule for bandwidth selection. Based on these estimates, we assess whether the ranking induced by $\tilde{\alpha}(x)$ (see Equation \ref{eq:mean_vector}) reflects meaningful distributional differences leveraging the metric $d$.

We design the following experiments to assess the applicability of our metric:

\begin{itemize}
    \item Section \ref{sec:quantify_stability}: we compute the stability of attribution vectors for the four explainer methods on a Random Forest (RF) model.
    \item Section \ref{sec:comparison}: we use $k$-stability to compare the robustness of different explainers across datasets and models (RF, Logistic Regression (LG), Decision Tree (DT), and Support Vector Machine (SVM)).
    \item Section \ref{sec:impact_runs}: we study the effect of the number of runs on stability estimation.
    \item Section \ref{sec:more_experiments}: we report additional experiments left to the appendices.
\end{itemize}

% todo: quitar
% \begin{table}[!h]
%     \centering
% \caption{Summary of the datasets used in the experiments. Preprocessing was performed on each one for analysis purposes.}
%     \label{tab:dataset_summary}
%     \begin{tabular}{lcc}
%     \toprule
%     Dataset & \# of Samples &  \# of Features\footnotemark \\
%     \midrule
%     Churn & 440832\footnotemark & 8 \\
%     Diabetes & 768 & 8 \\
%     Heart Disease & 303 & 13 \\
%     Mobile & 2000 & 20 \\
%     %Employee & 4653 & 7 \\
%     \bottomrule
% \end{tabular}  
% \end{table}
% \footnotetext[2]{Each dataset was preprocessed by removing unnecessary features and converting categorical variables into numerical form.}
% \footnotetext[3]{The amount of used samples was reduced to 1000, still maintaining the accuracy of the model, for computation reasons.}

\subsection{Analysis of Separability in the Attribution Vector}\label{sec:quantify_stability}

In this section, we analyze how separability evolves across the attribution vector. The first experiment computes the values of $d_1$ for consecutive feature pairs in the ranked attribution vector, across all four datasets. More precisely, given the attribution vector sorted in decreasing order of importance 
$(a_1, a_2, \dots, a_m)$, we compute $d_1(a_i, a_{i+1})$ for all $i = 1, \dots, m-1$. The resulting trends, shown in \autoref{fig:prog_analysis}, summarize how well the attribution distributions are separated at different positions in the ranking.

For the first two datasets, SHAP produces a constant value of $1$, reflecting the fact that its attributions are near-deterministic (meaning, due to the large sample size, the attribution values are identical across runs, up to decimal precision) and therefore maximally separable. In the more complex Heart and Mobile datasets, SHAP still exhibits higher separability overall, but its curve peaks toward the tail, suggesting increased overlap among the less relevant features. As expected, all attribution methods outperform the RAND baseline.
These results show how the framework allows a detailed examination of separability, highlighting differences across AMs and datasets.

\begin{figure}[!h]
    \centering
    \includegraphics[scale=0.2]{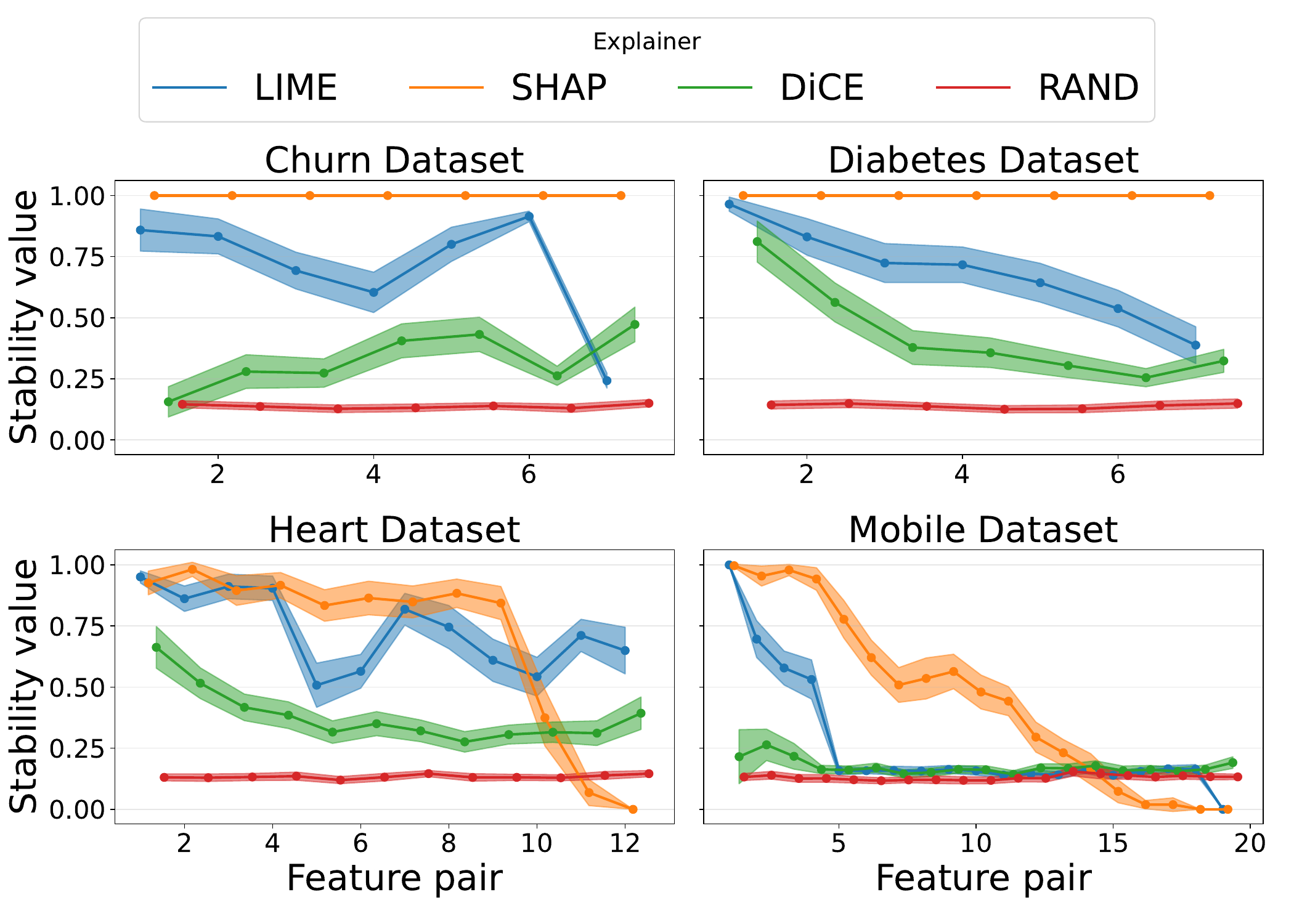}
    \caption{The stability values for the analyzed datasets calculated on the ordered feature pairs with the confidence interval $\pm 2\sigma/\sqrt{n}$, where $n=50$ is the number of instances. Here the x-axis refers to the feature pair $(\alpha_i, \alpha_{i+1})$. In the case of Churn and Diabetes, SHAP is achieves constantly the value of $1$.}
    \label{fig:prog_analysis}
\end{figure}

We extend the separability analysis by computing the pairwise distance between all features and averaging the resulting heatmaps over $50$ instances. Overall, \autoref{fig:heatmap_analysis} highlights that separability tends to increase with the positional distance between features, and the tail shows less separability than the more important variables. This graph provides an overview of the separability across all features, including those that are far apart in the ranking, thereby offering a more detailed view of the structure of the attribution vector.

\begin{figure}[!h]
    \centering

    \begin{subfigure}{0.48\textwidth}
        \centering
        \includegraphics[width=\linewidth]{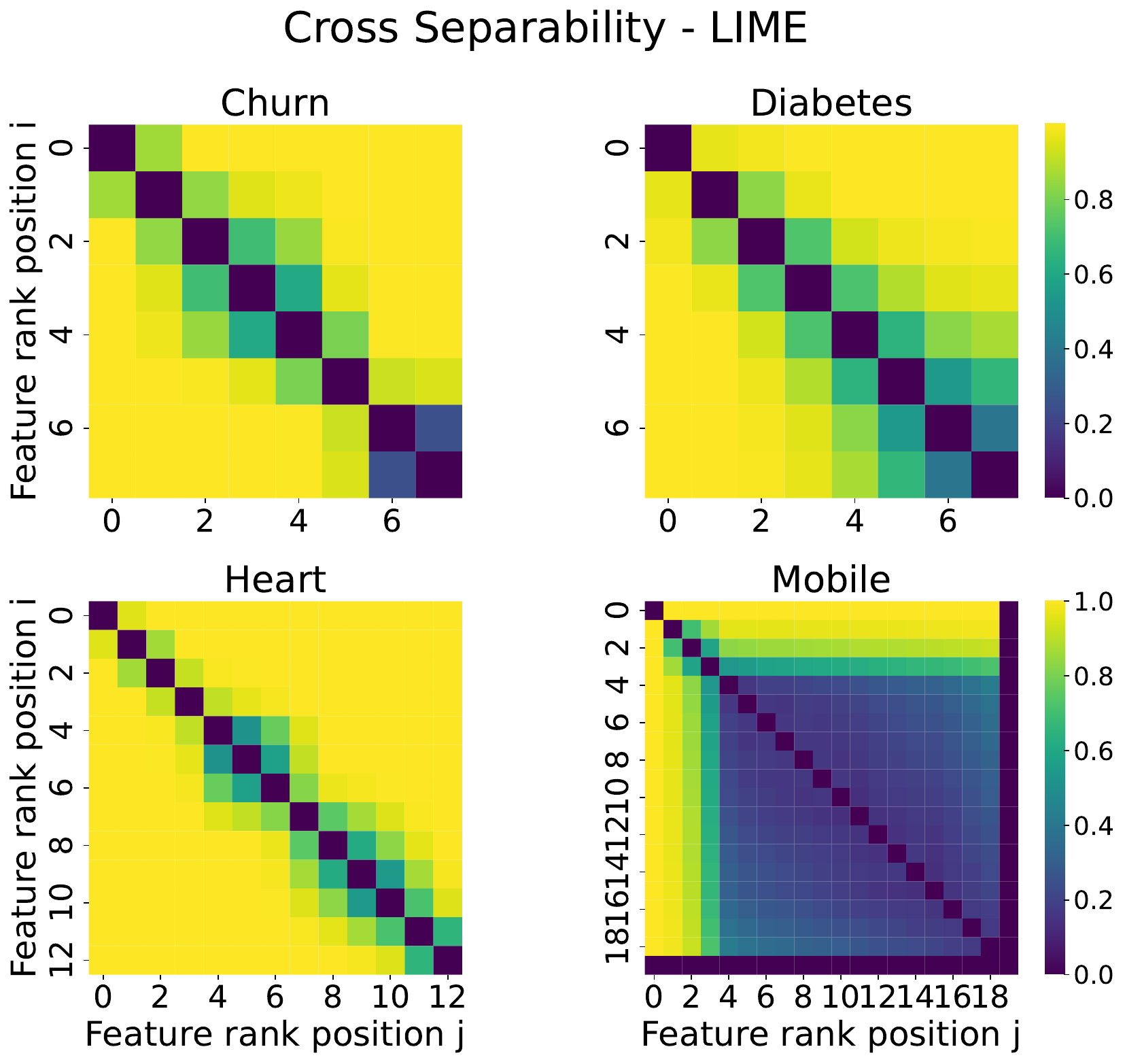}
    \end{subfigure}
    \hfill
    \begin{subfigure}{0.48\textwidth}
        \centering
        \includegraphics[width=\linewidth]{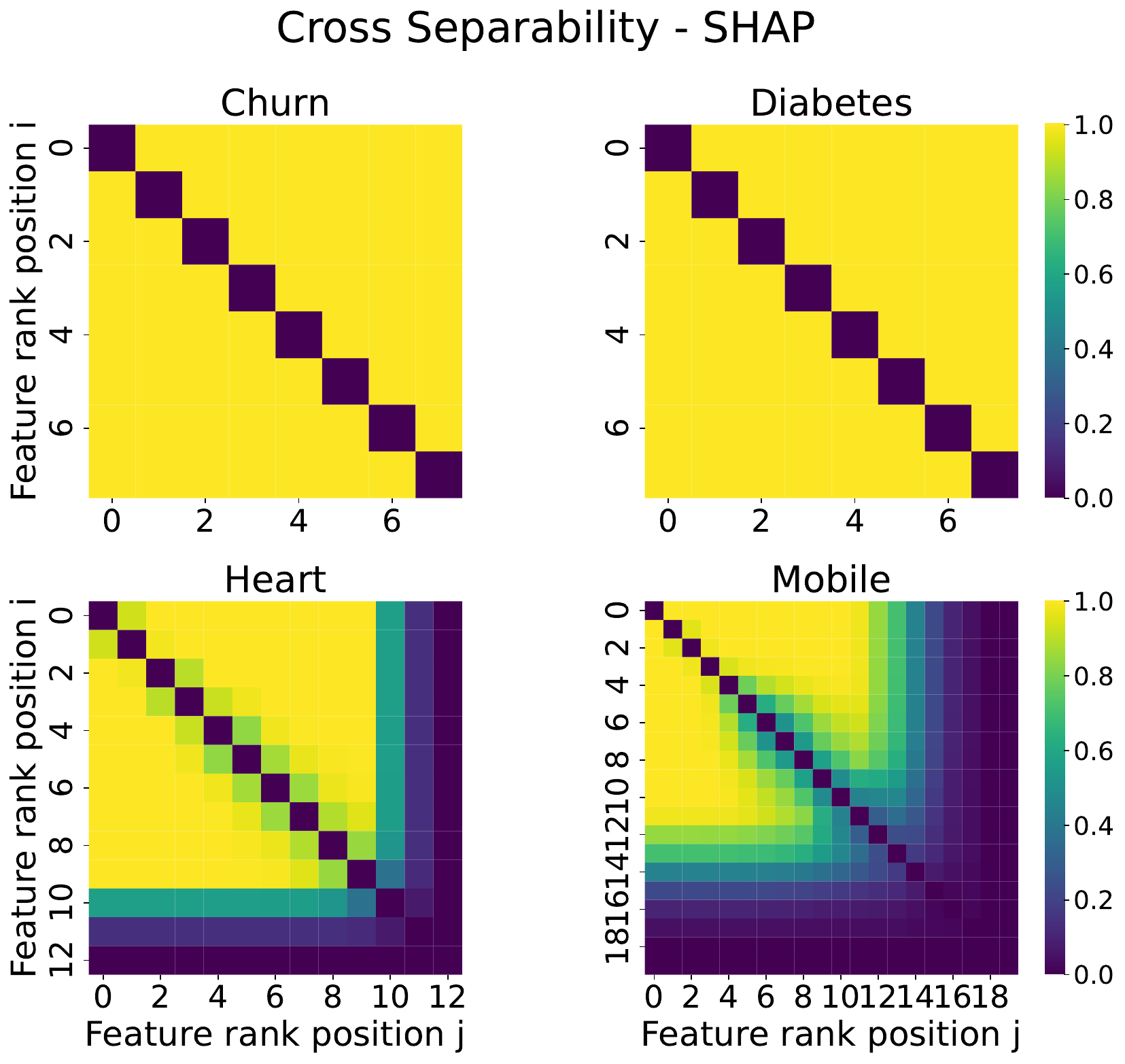}
    \end{subfigure}

    \vspace{0.3cm}

    \begin{subfigure}{0.48\textwidth}
        \centering
        \includegraphics[width=\linewidth]{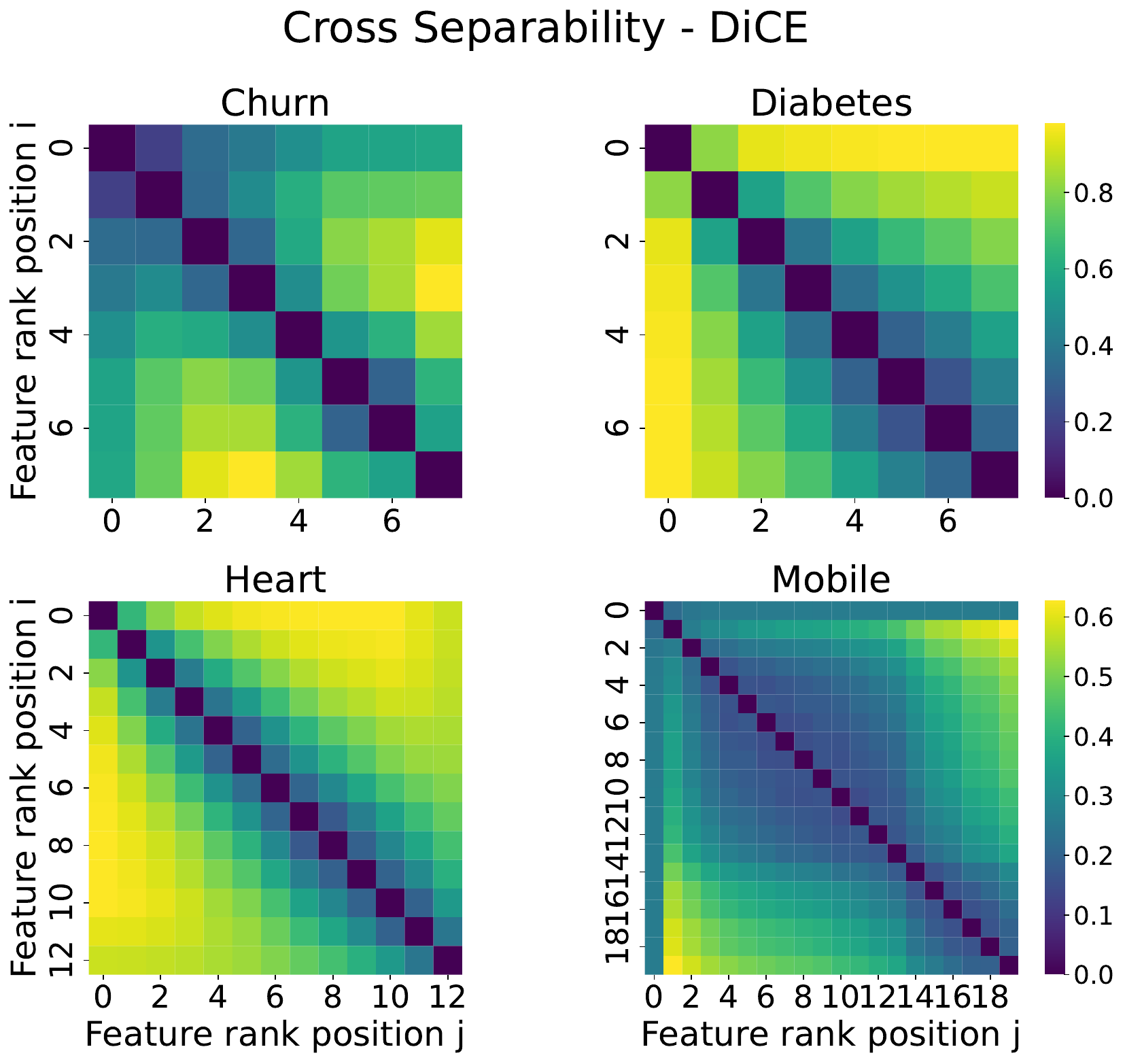}
    \end{subfigure}
    \hfill
    \begin{subfigure}{0.48\textwidth}
        \centering
        \includegraphics[width=\linewidth]{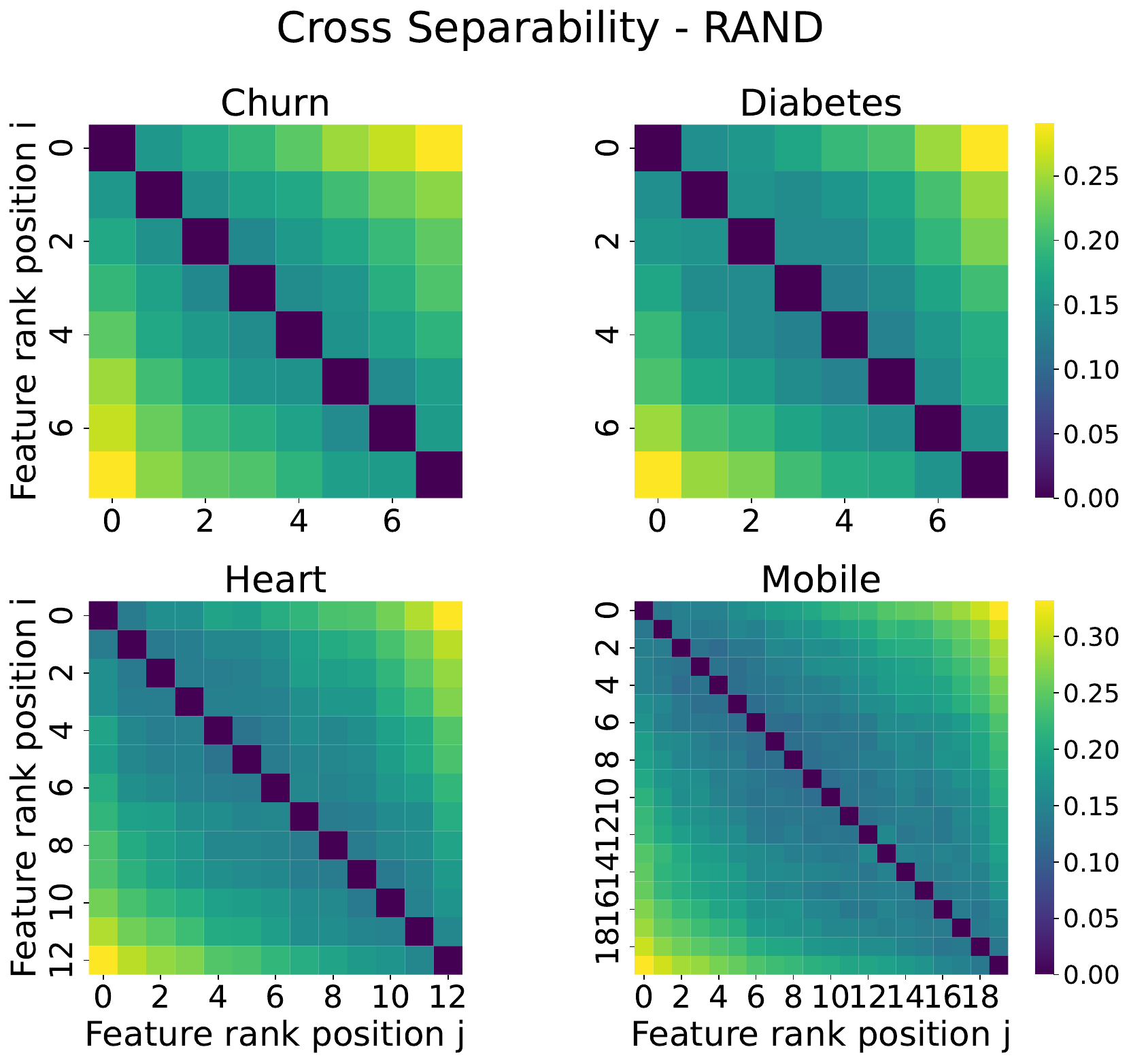}
    \end{subfigure}

    \caption{Average heatmap for the 50 points analyzed with different explanation methods. Each entry represents the symmetric distance $d(p_i,p_j)$. From left to right: LIME, SHAP; DiCE and RAND.}
    
    \label{fig:heatmap_analysis}
\end{figure}

\subsection{Comparing k-stability values of AMs} \label{sec:comparison}
Following \autoref{def:comparison}, we compare the four AMs across all datasets. In \autoref{tab:comparison_all} and \autoref{tab:comparison_all_svm} we report the results for the RF and SVM models (LG and DT left to the appendix). Both tables are consistent: first, we note that the last column being consistently equal to $1$ confirms that the RAND explainer serves as a reliable baseline, thereby validating the calibration of our analysis. Second, SHAP generally outperforms the other methods in the datasets under analysis, achieving higher proportions of larger $k$-stability values than LIME, DiCE, and RAND. LIME also performs better than DiCE and RAND. In contrast, DiCE consistently shows lower stability, reaching values close to the random baseline, particularly in the Mobile dataset.

Overall, our findings are consistent with previous studies reporting that SHAP tends to be more stable than LIME \cite{stability_perturbation,nayebi2023empirical}. However, according to the No Free Lunch Theorem for explainability, no AM can be expected to consistently outperform all others across datasets and models. For this reason, our goal is not to establish the superiority of a particular method, but rather to exemplify the application of a well-defined framework to quantitatively assess a specific property: the reliability of the feature ranking induced by its attribution scores.

\begin{table*}[!h]
    \centering
           \caption{Pairwise comparison of $k$-stability values across datasets for the RF model. 
    For each dataset we report the proportion of instances 
    for which the row method Lime (L), SHAP (S), DiCE (D) and RAND (R) achieves a $k$-stability greater than or equal to the column method.}
    \begin{tabular}{l|cccc|cccc|cccc|cccc}
        \toprule
        & \multicolumn{4}{c|}{\textbf{Churn}} 
        & \multicolumn{4}{c|}{\textbf{Diabetes}} 
        & \multicolumn{4}{c|}{\textbf{Heart}} 
        & \multicolumn{4}{c}{\textbf{Mobile}} \\
        \cmidrule(lr){2-5} \cmidrule(lr){6-9} \cmidrule(lr){10-13} \cmidrule(lr){14-17}
        & L & S & D & R 
        & L & S & D & R
        & L & S & D & R
        & L & S & D & R \\
        \midrule
        L & 1.00 & 0.00 & 0.98 & 1.00 
             & 1.00 & 0.00 & 0.96 & 1.00
             & 1.00 & 0.36 & 0.88 & 1.00
             & 1.00 & 0.16 & 1.00 & 1.00 \\
        S & \textbf{1.00} & \textbf{1.00} & \textbf{1.00} & \textbf{1.00} 
             & \textbf{1.00} & \textbf{1.00} & \textbf{1.00} & \textbf{1.00}
             & \textbf{0.74} & \textbf{1.00} & \textbf{0.98} & \textbf{1.00}
             & \textbf{0.96} & \textbf{1.00} & \textbf{1.00} & \textbf{1.00} \\
        D & 0.20 & 0.00 & 1.00 & 1.00 
             & 0.30 & 0.00 & 1.00 & 1.00
             & 0.32 & 0.16 & 1.00 & 1.00
             & 0.18 & 0.04 & 1.00 & 1.00 \\
        R & 0.18 & 0.00 & 0.96 & 1.00 
             & 0.10 & 0.00 & 0.34 & 1.00
             & 0.20 & 0.16 & 0.64 & 1.00
             & 0.00 & 0.02 & 0.82 & 1.00 \\
        \bottomrule
    \end{tabular}

    \label{tab:comparison_all}

\end{table*}

\begin{table*}[!h]
    \centering
           \caption{Pairwise comparison of $k$-stability values across datasets for the SVM model.}
    \begin{tabular}{l|cccc|cccc|cccc|cccc}
        \toprule
        & \multicolumn{4}{c|}{\textbf{Churn}} 
        & \multicolumn{4}{c|}{\textbf{Diabetes}} 
        & \multicolumn{4}{c|}{\textbf{Heart}} 
        & \multicolumn{4}{c}{\textbf{Mobile}} \\
        \cmidrule(lr){2-5} \cmidrule(lr){6-9} \cmidrule(lr){10-13} \cmidrule(lr){14-17}
        & L & S & D & R 
        & L & S & D & R
        & L & S & D & R
        & L & S & D & R \\
        \midrule
        L & 1.00 & 0.00 & 0.97 & 1.00 
             & 1.00 & 0.00 & 1.00 & 1.00
             & 1.00 & 0.30 & 0.83 & 1.00
             & 1.00 & 0.17 & 1.00 & 1.00 \\
        S & \textbf{1.00} & \textbf{1.00} & \textbf{1.00} & \textbf{1.00} 
             & \textbf{1.00} & \textbf{1.00} & \textbf{1.00} & \textbf{1.00}
             & \textbf{0.80} & \textbf{1.00} & \textbf{0.97} & \textbf{1.00}
             & \textbf{0.97} & \textbf{1.00} & \textbf{1.00} & \textbf{1.00} \\
        D & 0.17 & 0.00 & 1.00 & 1.00 
             & 0.23 & 0.00 & 1.00 & 1.00
             & 0.37 & 0.13 & 1.00 & 1.00
             & 0.17 & 0.03 & 1.00 & 1.00 \\
        R & 0.17 & 0.00 & 0.60 & 1.00 
             & 0.13 & 0.00 & 0.63 & 1.00
             & 0.23 & 0.07 & 0.63 & 1.00
             & 0.00 & 0.02 & 0.73 & 1.00 \\
        \bottomrule
    \end{tabular}

    \label{tab:comparison_all_svm}

\end{table*}

\subsection{The impact of the number of runs}\label{sec:impact_runs}

Finally, we analyze how the distribution of $d_1$ for the ordered feature pairs evolves when varying the number of runs of the method. In \autoref{fig:run_analysis}, we report the stability values along with confidence intervals $\pm 2\sigma/\sqrt{n}$. The results indicate that the proposed metric exhibits a \textbf{limited sensitivity to the number of runs} for LIME, SHAP and DiCE, suggesting that reliable estimates can be obtained without excessive computational cost. A mild deviation is observed in the Mobile dataset when using only 5 runs, where the distances between feature distributions appear overestimated. Only in the case of RAND more runs are necessary, but this is due to the random behavior of the explainer

\begin{figure}[!h]
    \centering

    \begin{subfigure}{0.48\textwidth}
        \centering
        \includegraphics[width=\linewidth]{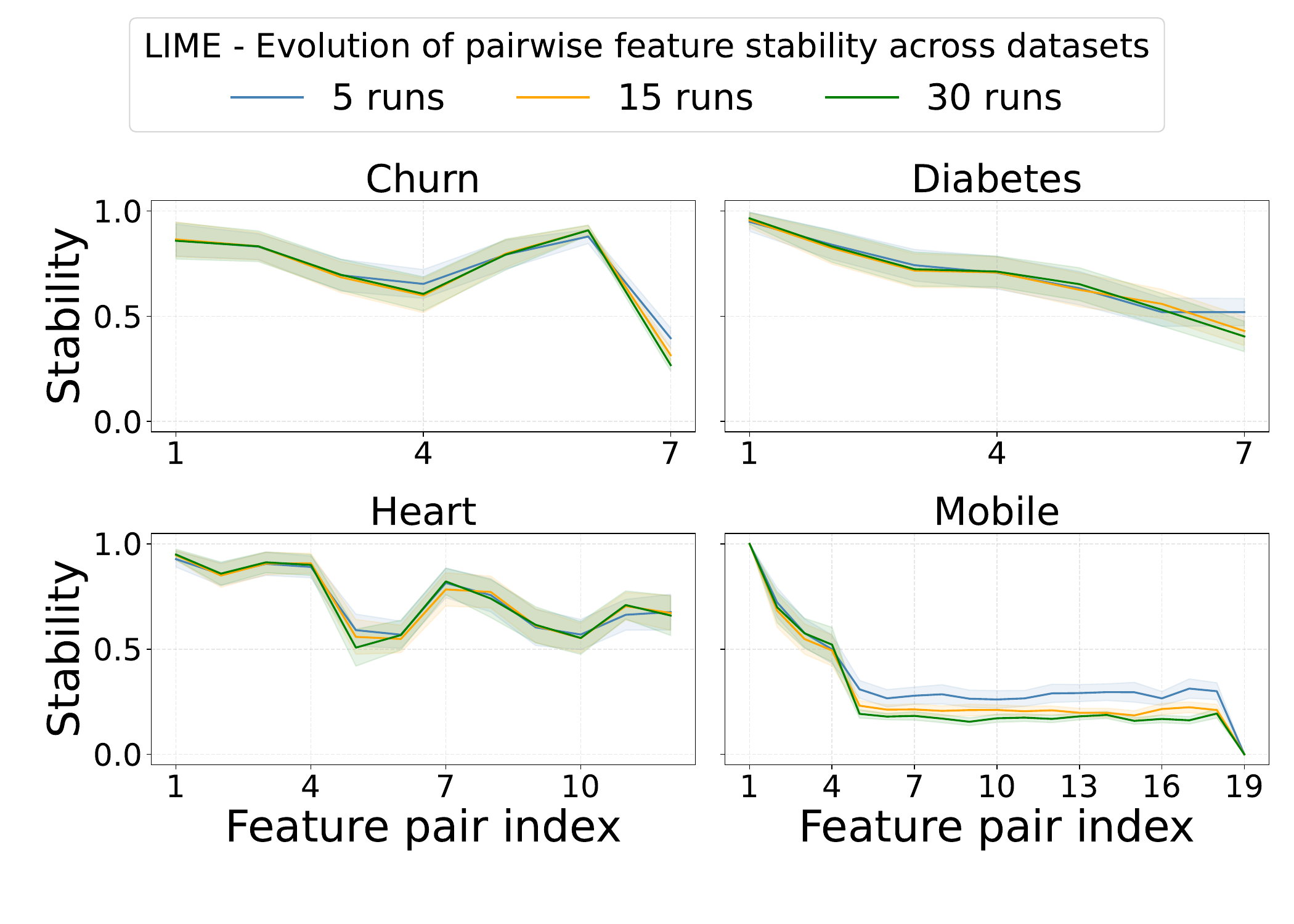}
    \end{subfigure}
    \hfill
    \begin{subfigure}{0.48\textwidth}
        \centering
        \includegraphics[width=\linewidth]{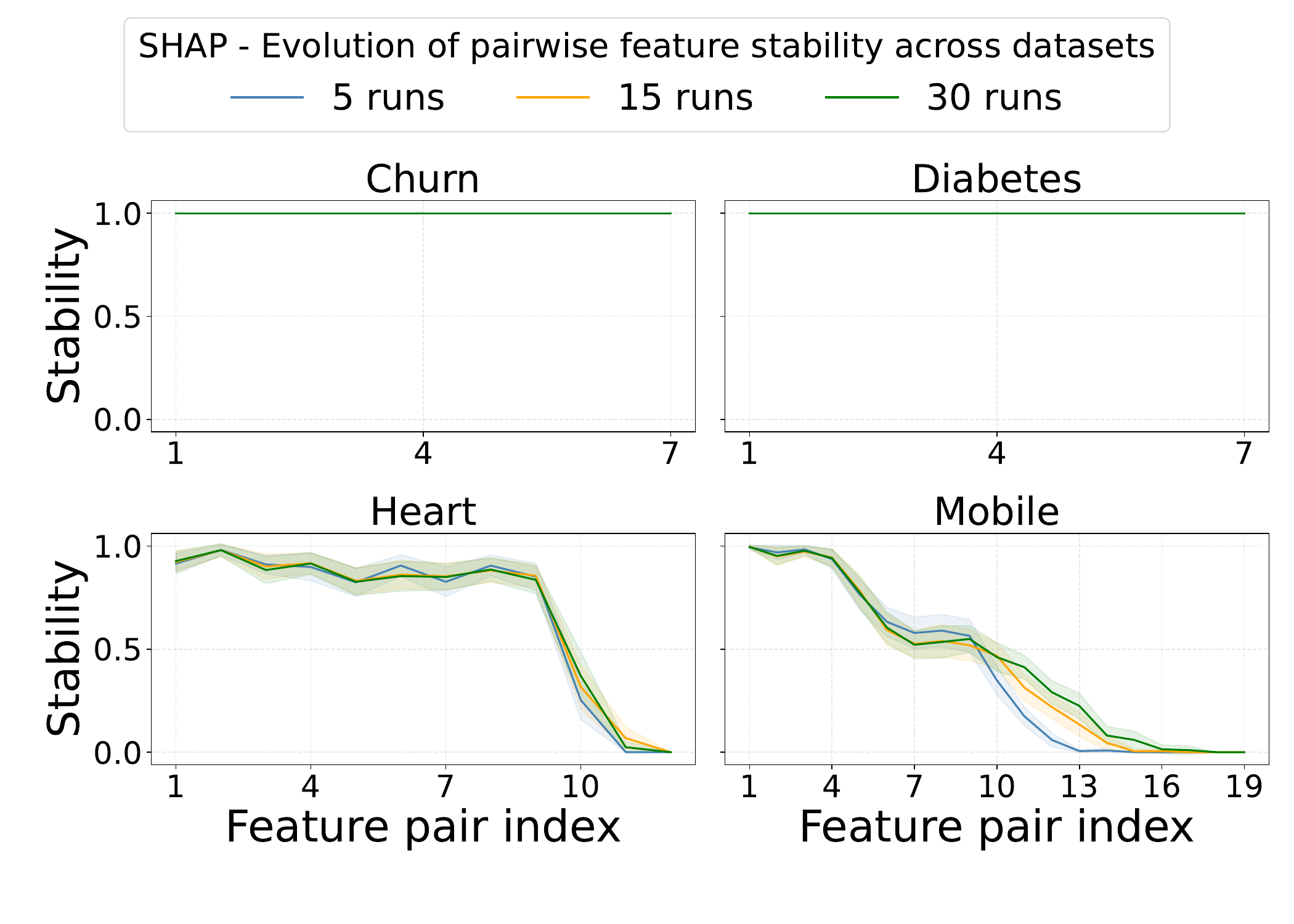}
    \end{subfigure}

    \vspace{0.3cm}

    \begin{subfigure}{0.48\textwidth}
        \centering
        \includegraphics[width=\linewidth]{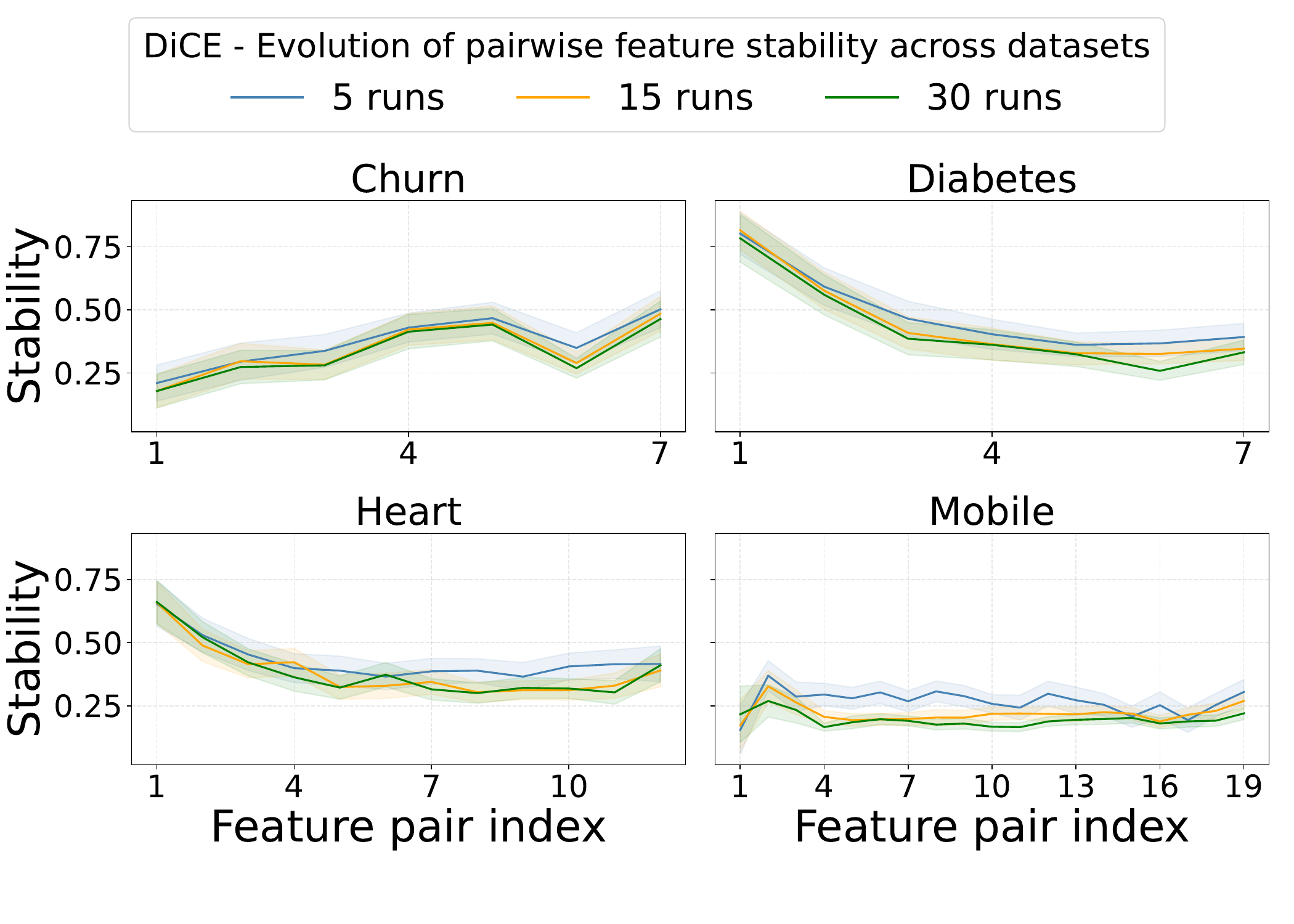}
    \end{subfigure}
    \hfill
    \begin{subfigure}{0.48\textwidth}
        \centering
        \includegraphics[width=\linewidth]{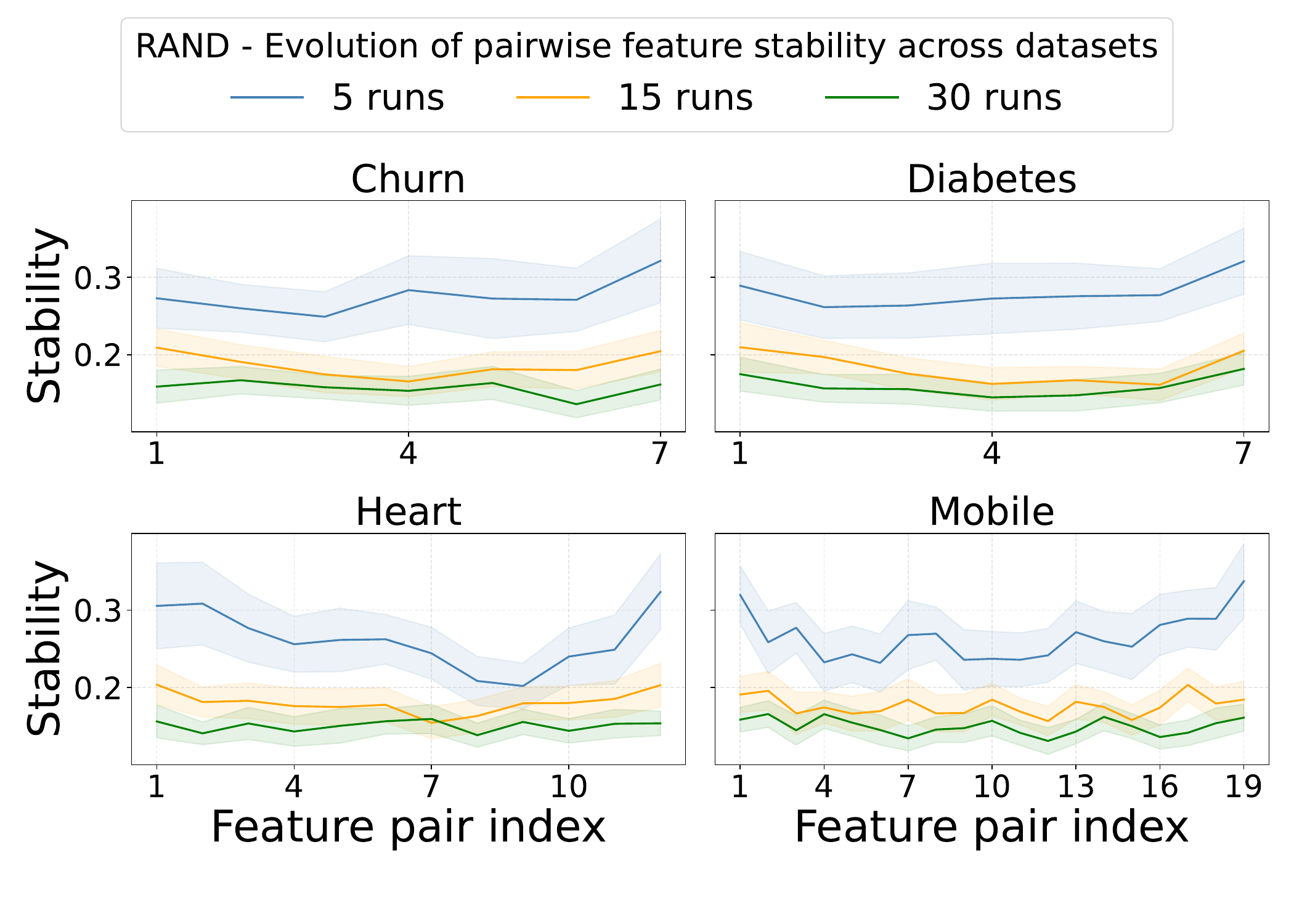}
    \end{subfigure}

    \caption{Stability values across different numbers of runs for LIME (top-left), SHAP (top-right), DiCE (bottom-left), and RAND (bottom-right) on the considered datasets. For Churn and Diabetes, SHAP yields perfect stability since the attribution vectors were quasi-deterministic.}
    
    \label{fig:run_analysis}
\end{figure}

\subsection{Other Experiments}\label{sec:more_experiments}
We have decided to keep the main body of the paper focused on the key experiments and insights. However, for the sake of completeness, we have conducted numerous additional experiments left to the appendix, which we summarize briefly here:
\begin{itemize}
    \item \textbf{Group-based stability} (Appendix \ref{sec:appenix_GB}): in many cases, the assignment vector is binary, i.e. $\alpha(x)\in [0,1]^n$. In this scenario, we show how our strategy can be used to identify a set of features that can be distinguished from the rest.
    % \item \textbf{Comparison of AM across models} (Section \ref{sec:appendix_additional}): we complete Section \ref{sec:comparison} by computing the pairwise comparison of AMs for a Decision Tree (DT) and a Support Vector Machine (SVM) model, indicating SHAP as the most stable according to our metric. 
    \item \textbf{Hyperparameter discussion} (Appendix \ref{sec:hyperparameter_appendix}): we analyze the impact of our framework's parameters; specifically, the threshold $l$ and the KDE kernel. We show how reducing $l$ increases the k-stability values (as we require less separability between distributions), as expected, whilst highlighting Silverman’s rule as a robust strategy for inferring the distributions.
    \item \textbf{Synthetic validation of mathematical results} (Appendix \ref{sec:sythetic_empirical_validation}): we complement Section \ref{sec:validation} by performing the analysis of k-stability in synthetic settings with $5$ and $10$ features. In general, the results confirm that greater separability of features leads to higher values of k-stability.
\end{itemize}

\section{Limitations and Further Details} \label{sec:limitations}
In this work we proposed an alternative method to quantify the stability of AMs, which in turn provides a meaningful way to compare explainers, as shown in Section \ref{sec:comparison}. Our contribution is not to establish which AM is superior, but rather to introduce a comparative metric that enables more nuanced evaluations on a specific aspect of the instability. It is important to stress that no AM can be considered universally more appropriate than others \cite{lipton2018mythos,Molnar}. Instead, the community has emphasized the need for multi-level and in-depth comparisons \cite{belaid2023compare,bodria2023benchmarking,agarwal2022openxai}, and our work contributes to this perspective. 

Naturally, our approach comes with some limitations.
First, our empirical evaluation focuses on controlled settings, namely classification tasks on standard tabular benchmark datasets. While this choice allows us to isolate and study the properties of the proposed metric, future work should investigate its applicability to other data modalities, such as images, text, and time series.

Second, the framework requires selecting both a density estimation procedure and the threshold parameter $l$. Although we discuss these design choices and their impact in the appendix, their selection may influence the resulting stability estimates.

Finally, as discussed in Appendix \ref{sec:theoretical_appendix}, low $k$-stability values do not necessarily imply that an attribution method is unfaithful. Moreover, even when an AM faithfully reflects the model's decision process, stochastic components in the explanation procedure may induce variability across runs, leading to overlapping attribution distributions and consequently lower stability scores. Therefore, $k$-stability alone cannot capture the overall quality of an explanation, but rather should be interpreted as a measure of ranking reliability and eventually combined with other metrics.

\section{Conclusions}
In this work, we have formalized a novel method for measuring the stability of AM rankings, providing an additional dimension for explanation quality. Our experiments show that our method achieves consistent results even with a low number of explainer executions, and highlight SHAP as the most k-stable explainer for the datasets tested.

The strategy we have defined is not limited to measuring ranking stability. It can also be employed to study feature-pair separability within an attribution vector and to identify groups of features that are significantly more prominent than the rest (group-based explanations).

This flexibility allows to measure explainer stability from a variety of perspectives, which is fundamental for determining the right explainer for any given task. In line with the XAI community, which has highlighted the importance of a multi-dimensional evaluation of explanation quality, our approach provides a principled methodology to measure ranking reliability, thereby complementing explainer assessment based on other considerations, such as faithfulness or simplicity.

\section*{Acknowledgements}

Eddie Conti has been partially supported by the predoctoral grant FI-STEP (2025 STEP 00108) from the Research and University Department of the Generalitat de Catalunya and cofunded by the European Social Fund Plus. \'Alvaro Parafita acknowledges his AI4Science fellowship within the “Generacion D” initiative by Red.es, Ministerio para la Transformación Digital y de la Función Pública, for talent attraction (C005/24-ED CV1), funded by NextGenerationEU through PRTR. Axel Brando received funding from the Horizon Europe Programme under the AI4DEBUNK Project (https://www.ai4debunk.eu), grant agreement num. 101135757.
% In conclusion, the XAI community has highlighted the need for evaluation from multiple perspectives in order to assess the quality of an explanation. In line with this, our metric aims to add a new dimension and address a specific need: measuring ranking reliability.

%Moreover, we provide a theoretical and experimental analysis highlighting how the degree of separation among the true feature importance values plays a key role in shaping our notion of stability, adding to the multitude of facets involved in the variability of stochastic AMs. In conclusion, this contribution provides a new and complementary perspective for analyzing, evaluating and comparing explainers in greater depth, while also linking stability with Uncertainty Quantification as a promising new direction.  

% the environments 'definition', 'lemma', 'proposition', 'corollary',
% 'remark', and 'example' are defined in the LLNCS documentclass as well.
%

% 6700 characters
\bibliographystyle{splncs04}
\bibliography{mybib}
%% Note that this preceding line implies that you store your BibTeX references in a file called 'mybibliography.bib'. If you instead store your references in a file with a different name, for instance 'references.bib', the preceding line should read '\bibliography{references}'. Whatever you do, DO NOT put the file name extension .bib inside the \bibliography command; this will trip up LaTeX compilers. 
%
% If you do not want to use BibTeX, you can also type up the bibliography exactly as you see fit, using the following structure:
% Note that this number 8 reserves an amount of space (equal to the natural width of the given number) for the label of your references; if you have more than 9 references, you will want to change this number to 18. If you have more than 19 references, this number is best changed to 88. If you have more than 99 references, I salute you.
\setcounter{section}{0}
\renewcommand\thesection{\Alph{section}}
\section{Supplementary Material}
\subsection{Group-based Stability} \label{sec:appenix_GB}
In this section we extend the notion of separability of attribution scores to \emph{binary explanations}, where an AM identifies a set of relevant features rather than producing only a ranking. Let $G \subseteq [n]$ denote the selected features and $\bar{G}$ its complement. We define the \emph{group-based stability} of $G$ as:
\[
\mathrm{GroupStab}(G)=
f\bigl(\{ d(p_i,p_j) : i \in G,; j \in \bar{G} \}\bigr),
\]
where $f$ is an aggregation function such as the mean, median, or minimum. This formulation captures different notions of separability between relevant and non-relevant features.

The framework naturally supports several explanation strategies, including top-$k$ features, threshold-based selection, and cumulative attribution criteria. In the special case where $G$ contains the top-$k$ ranked features, we recover a condition analogous to $k$-stability:
\[
d(p_i,p_j)\ge l
\qquad
\forall i\in G,; \forall j\in\bar{G},
\]
for a threshold $l\in[0,1]$. This allows identifying the largest value of $k$ for which the selected group remains sufficiently separable.

In \autoref{fig:group_stability} we apply this framework to LIME explanations in the case of the RF model. For each instance and each value of $k$, we verify whether the top-$k$ features satisfy the separability condition with $l=0.9$. The result is binary: a value of $1$ is assigned if the condition holds and $0$ otherwise. The resulting curves report, across 50 instances, the proportion of explanations certified as separable.

The analysis highlights several trends. First, evaluating separability as a function of $k$ provides a principled way to determine the size of the explanation according to the desired level of stability. Second, datasets such as Churn and Heart exhibit non-monotonic certification curves, suggesting that while the precise ordering of important features may be unstable, groups of features can still remain jointly separable (e.g., in Churn with $6$ features we have a certification rate $\approx 60\%$, while with $4$ of $\approx 30\%$). Finally, the Mobile dataset shows consistently low certification rates, likely due to the intrinsic complexity of the dataset and the difficulty of the explainer in identifying clearly separable features.

\begin{figure}[!h]
    \centering
    \includegraphics[scale=0.23]{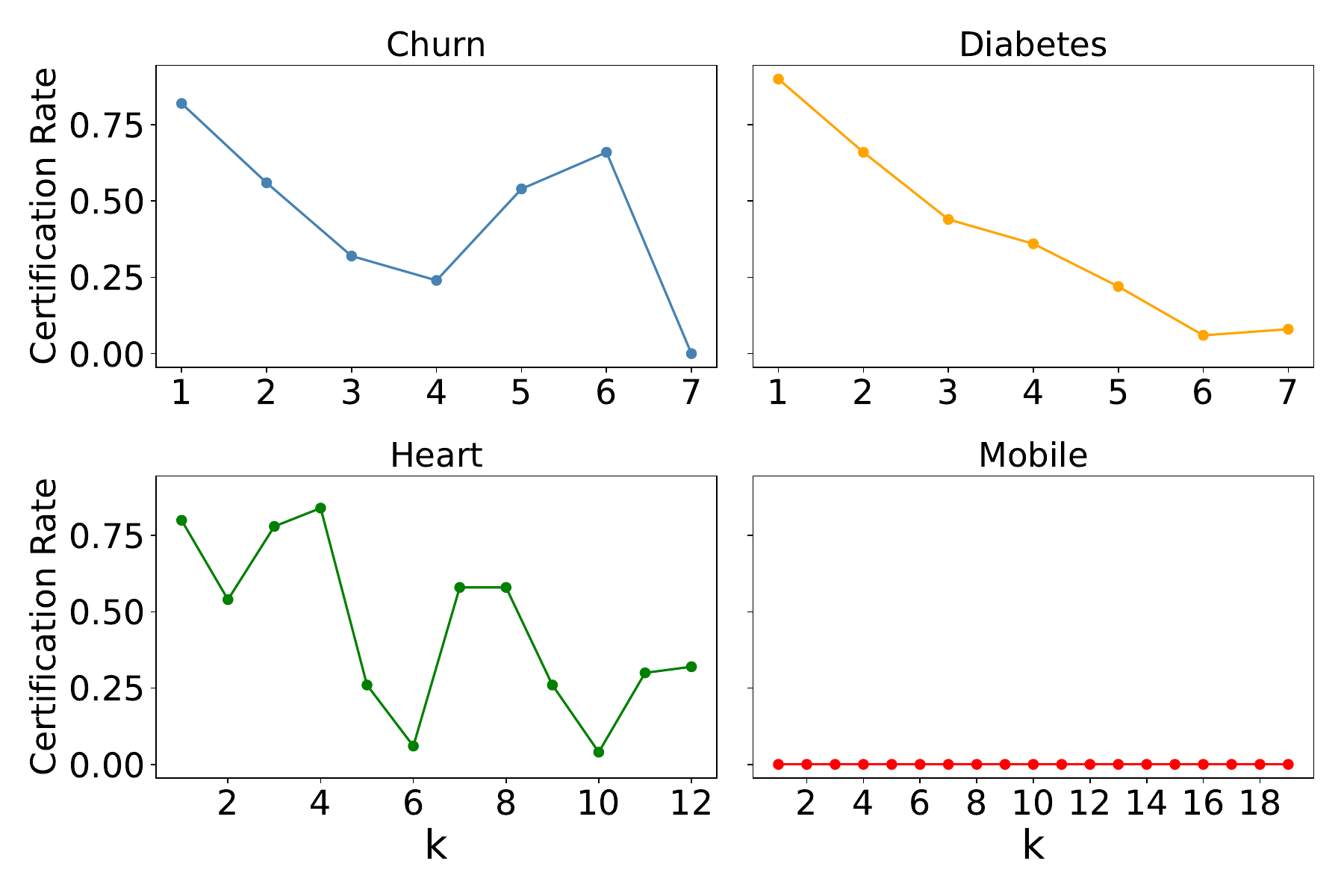}
    \caption{Proportion of instances in which the LIME top-$k$ features form a separable group according to the group-stability condition.}
    \label{fig:group_stability}
\end{figure}

\subsection{Additional comparison experiments} \label{sec:appendix_additional} 
In this section, we supplement the experiments comparing attribution methods. In particular, we find that SHAP produces more separable rankings for DT and LG as well (\autoref{tab:comparison_all_dt}, \autoref{tab:comparison_all_lr}), achieving higher values in our stability metric 

\begin{table*}[!h]
    \centering
           \caption{Pairwise comparison of $k$-stability values across datasets for the DT model.}
    \begin{tabular}{l|cccc|cccc|cccc|cccc}
        \toprule
        & \multicolumn{4}{c|}{\textbf{Churn}} 
        & \multicolumn{4}{c|}{\textbf{Diabetes}} 
        & \multicolumn{4}{c|}{\textbf{Heart}} 
        & \multicolumn{4}{c}{\textbf{Mobile}} \\
        \cmidrule(lr){2-5} \cmidrule(lr){6-9} \cmidrule(lr){10-13} \cmidrule(lr){14-17}
        & L & S & D & R 
        & L & S & D & R
        & L & S & D & R
        & L & S & D & R \\
        \midrule
        L & 1.00 & 0.00 & 0.93 & 1.00 
             & 1.00 & 0.00 & 1.00 & 1.00
             & 1.00 & 0.23 & 1.00 & 1.00
             & 1.00 & 0.23 & 1.00 & 1.00 \\
        S & \textbf{1.00} & \textbf{1.00} & \textbf{1.00} & \textbf{1.00} 
             & \textbf{1.00} & \textbf{1.00} & \textbf{1.00} & \textbf{1.00}
             & \textbf{0.90} & \textbf{1.00} & \textbf{0.97} & \textbf{1.00}
             & \textbf{0.93} & \textbf{1.00} & \textbf{1.00} & \textbf{1.00} \\
        D & 0.50 & 0.00 & 1.00 & 1.00 
             & 0.20 & 0.00 & 1.00 & 1.00
             & 0.27 & 0.03 & 1.00 & 1.00
             & 0.07 & 0.00 & 1.00 & 1.00 \\
        R & 0.30 & 0.00 & 0.60 & 1.00 
             & 0.17 & 0.00 & 0.63 & 1.00
             & 0.20 & 0.03 & 0.90 & 1.00
             & 0.00 & 0.00 & 0.80 & 1.00 \\
        \bottomrule
    \end{tabular}

    \label{tab:comparison_all_dt}

\end{table*}

\begin{table*}[!h]
    \centering
           \caption{Pairwise comparison of $k$-stability values across datasets for the LR model.}
    \begin{tabular}{l|cccc|cccc|cccc|cccc}
        \toprule
        & \multicolumn{4}{c|}{\textbf{Churn}} 
        & \multicolumn{4}{c|}{\textbf{Diabetes}} 
        & \multicolumn{4}{c|}{\textbf{Heart}} 
        & \multicolumn{4}{c}{\textbf{Mobile}} \\
        \cmidrule(lr){2-5} \cmidrule(lr){6-9} \cmidrule(lr){10-13} \cmidrule(lr){14-17}
        & L & S & D & R 
        & L & S & D & R
        & L & S & D & R
        & L & S & D & R \\
        \midrule
        L & 1.00 & 0.00 & 0.97 & 1.00 
             & 1.00 & 0.00 & 1.00 & 1.00
             & 1.00 & 0.13 & 0.90 & 1.00
             & 1.00 & 0.20 & 1.00 & 1.00 \\
        S & \textbf{1.00} & \textbf{1.00} & \textbf{1.00} & \textbf{1.00} 
             & \textbf{1.00} & \textbf{1.00} & \textbf{1.00} & \textbf{1.00}
             & \textbf{0.97} & \textbf{1.00} & \textbf{0.97} & \textbf{1.00}
             & \textbf{0.90} & \textbf{1.00} & \textbf{1.00} & \textbf{1.00} \\
        D & 0.20 & 0.00 & 1.00 & 1.00 
             & 0.30 & 0.00 & 1.00 & 1.00
             & 0.60 & 0.10 & 1.00 & 1.00
             & 0.07 & 0.03 & 1.00 & 1.00 \\
        R & 0.13 & 0.00 & 0.60 & 1.00 
             & 0.20 & 0.00 & 0.63 & 1.00
             & 0.60 & 0.16 & 0.90 & 1.00
             & 0.00 & 0.00 & 0.80 & 1.00 \\
        \bottomrule
    \end{tabular}

    \label{tab:comparison_all_lr}

\end{table*}

\subsection{Hyperparameter discussion}  \label{sec:hyperparameter_appendix}

This section briefly analyzes the sensitivity of the proposed k-stability measure to two key hyperparameters: the density estimator and the threshold $l$, using the Diabetes dataset and LIME as an illustrative case.

Regarding the density estimator, the use of Gaussian kernels with Silverman’s rule is considered a robust and standard choice. Other bandwidth values are tested in \autoref{fig:kstability_comparison}. Although different bandwidths produce slightly different inferred distributions of attribution values, the resulting k-stability remains very similar across settings, except for $bandwidth=1$, which yields lower stability. Together with previous findings showing that fewer samples are sufficient to separate attribution values, this suggests that k-stability is not strongly dependent on the modeling capacity of the density estimator, supporting Silverman’s rule as a safe default.

The threshold $l$ has a more pronounced effect. It determines how large the distributional distance between feature attributions must be to ensure a stable ordering. Lower values of $l$ allow more overlap between feature importance distributions, potentially leading to less reliable rankings, while higher values impose stricter separability requirements. There is no universally optimal choice of $l$, as it depends on the application context (e.g., more conservative settings in high-stakes domains such as medical diagnosis, versus more relaxed exploratory analyses). In this work, $l=0.9$ is used. In \autoref{fig:kstability_comparison}, different threshold values are compared, showing that lower $l$ values lead to higher k-stability, since the stability condition is more easily satisfied.

\begin{figure}[!h]
    \centering

    \begin{subfigure}{0.48\textwidth}
        \centering
        \includegraphics[width=\linewidth]{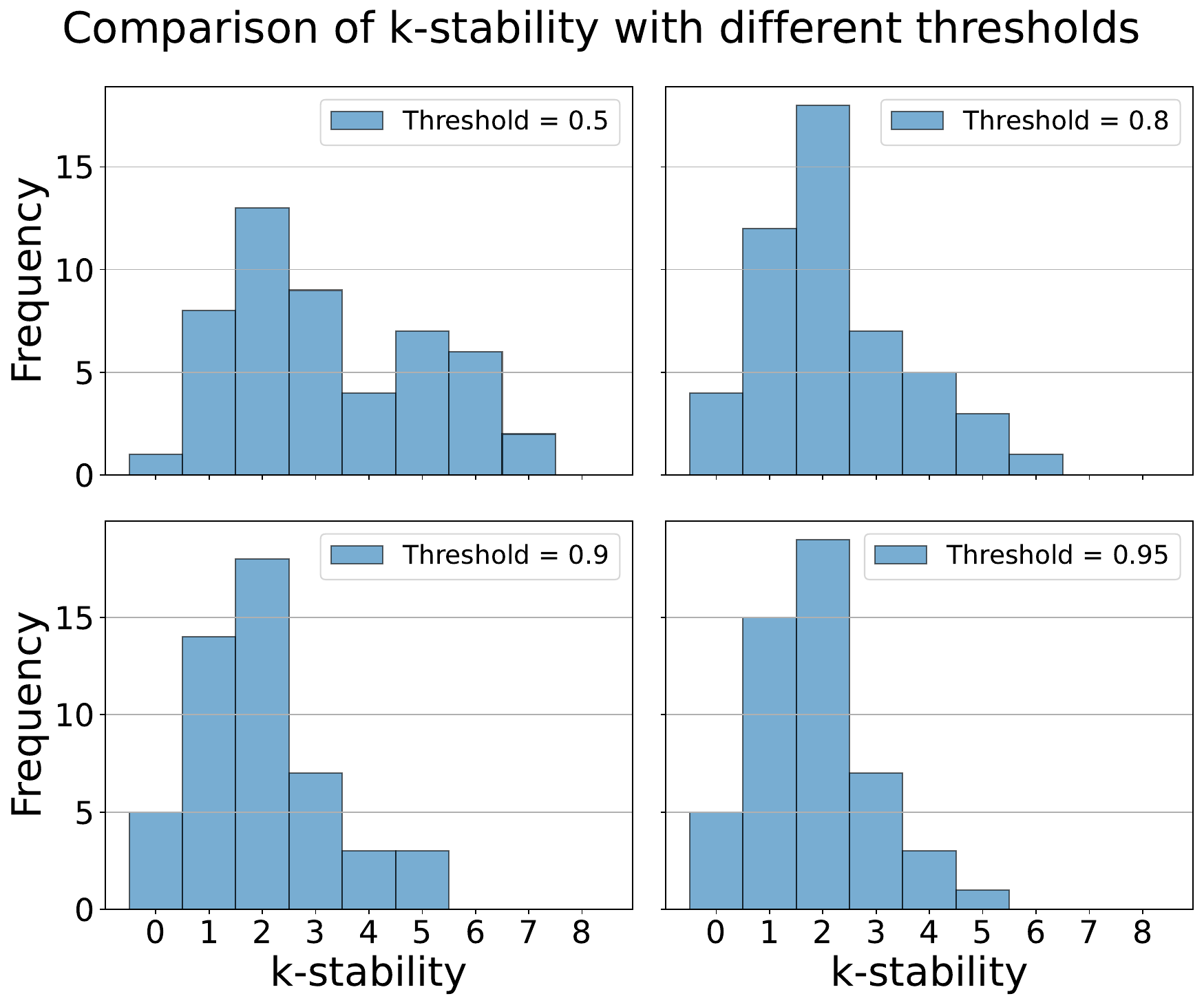}
    \end{subfigure}
    \hfill
    \begin{subfigure}{0.48\textwidth}
        \centering
        \includegraphics[width=\linewidth]{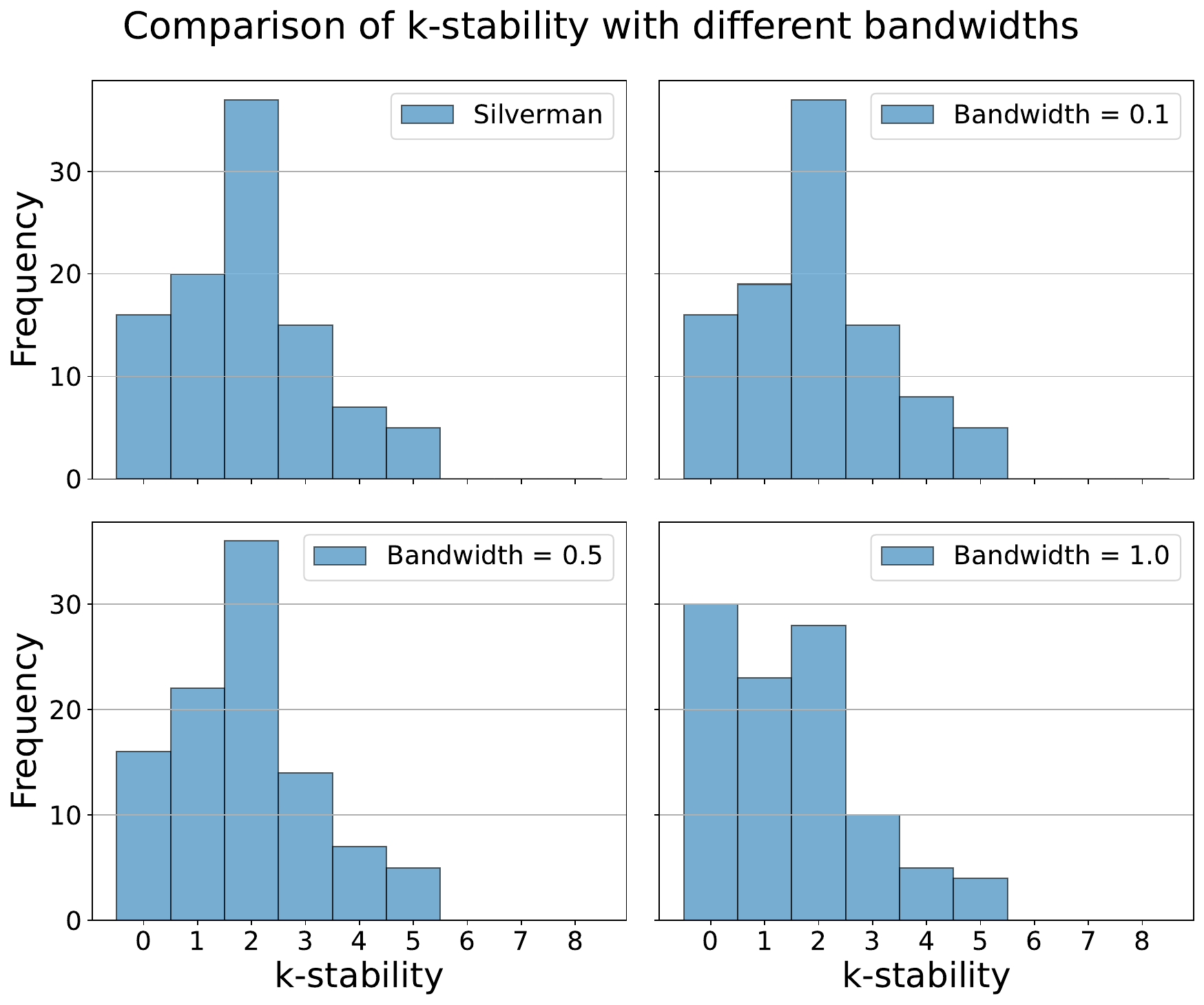}
    \end{subfigure}

    \caption{Distributional comparison of k-stability values under different parameter settings. On the left, different threshold values $l$ ($0.5$, $0.8$, $0.9$, and $0.95$) are considered. On the right, different bandwidth values are compared: Silverman's rule, $0.1$, $0.5$, and $1$.}
    
    \label{fig:kstability_comparison}
\end{figure}

\subsection{Theoretical Guarantees and Limitations} \label{sec:theoretical_appendix}
We begin by establishing a preliminary result that guarantees the well-posedness of our stability metric in the ideal scenario where AMs converge to deterministic feature importances. Throughout the following discussion, we will assume that the distribution of attribution scores is approximately Gaussian. This hypothesis not only enables us to derive analytical results for the proposed metric, but it is also consistent with assumptions commonly adopted in the related literature \cite{covert2021improving,zhou2021s,goldwasser2024stabilizing,goldwasser2024statistical,agarwal2022openxai,aas2021explaining}.
Moreover, in the following discussion, we evaluate explainers with their hyperparameters fixed; this is essential to model the variance of the attribution scores, since increasing the sampling parameters (e.g., in SHAP, by raising the number of coalitions or background points) would eventually lead to quasi-deterministic attributions, thus removing the variability that our analysis aims to characterize.

We recall the definition of the metric $d$:

\begin{equation} \label{eq:distance}
d(p,q) = 1 - \frac{\int_{\text{supp}(p) \cap \text{supp}(q)} \min(p(x),q(x)) dx}{\int_{\text{supp}(p) \cup \text{supp}(q)} \max(p(x),q(x)) dx},
\end{equation}

\begin{proposition}[Well-posedness in the ideal case] \label{prop:well-posedness}
Let $x\in\mathbb{R}^n$ be a point to be explained with ground-truth importances 
$f_1>\cdots>f_n$. Consider an AM $\alpha$ with a stochastic component
that produces, at each run $s=1,\ldots,m$, an attribution vector 
$\alpha^{(s)}=\{\alpha_1^{(s)},\ldots,\alpha_n^{(s)}\}$. 
Denote by $p_{i,m}$ the empirical probability density function associated with the values 
$\{\alpha_i^{(s)}\}_{s=1}^m$.  

Assume that, for each $i$, $p_{i,m}$ converges in distribution to the Dirac measure 
$\delta_{f_i}$ as $m \to \infty$. Then, for the distance metric $d(\cdot,\cdot)$ defined
in \eqref{eq:distance}, we have
\[
d(p_{i,m},p_{i+1,m}) \to 1 \quad \text{as } m \to \infty,
\quad \forall i \in \{1,\ldots,n-1\}.
\]
\end{proposition}

\begin{proof}
Since $p_{i,m} \xrightarrow{d} \delta_{f_i}$ and $p_{i+1,m} \xrightarrow{d} \delta_{f_{i+1}}$
as $m \to \infty$, with $f_i \neq f_{i+1}$, the two limiting measures are mutually singular.
By definition of the distance $d(\cdot,\cdot)$, since they are separated, this implies
\[
d(\delta_{f_i},\delta_{f_{i+1}}) = 1.
\]
By the continuity of $d$, we obtain
\[
d(p_{i,m},p_{i+1,m}) \to d(\delta_{f_i},\delta_{f_{i+1}}) = 1,
\]
which proves the claim.
\end{proof}

This result confirms that our metric is well-posed in the sense that, if an AM
were perfectly faithful and its stochasticity vanished asymptotically, then it would achieve
the maximum possible stability under the $d$ metric. We will focus later on the case of having a couple of feature importances $f_i,f_j$ to be equal and analyze how the metric $d$ behaves in this situation.

\begin{remark}[Limitations of k-stability]
The proposed metric has two complementary limitations that clarify its interpretation.

First, k-stability may attain its maximum value even for completely unfaithful attribution methods. In particular, even when the learned attributions converge to arbitrary but distinct constants (i.e., degenerate distributions), the resulting pairwise separability can be maximal, leading to perfect stability. This shows that k-stability alone cannot distinguish faithful from unfaithful explanations, and should not be interpreted as a direct measure of correctness.

Second, even in the case of faithful and unbiased attribution methods, stability is fundamentally constrained by stochasticity and by the intrinsic closeness of the true feature importances. When the separation between ground-truth values is small relative to the estimator variance, distributional overlap can prevent the stability condition from being satisfied, even asymptotically.

%Overall, these observations highlight that k-stability captures the separability of attribution distributions rather than their absolute faithfulness, and should be interpreted as one component within a broader evaluation framework for explanation quality.
\end{remark}

Let us now formalize the previous remark in a simplified setting. Consider an unbiased and faithful AM $\alpha: \mathbb{R}^n \to [0,1]^n$ (which requires to rescale in the case of DiCE, or taking the absolute value in LIME and SHAP). We focus on two features $f_i,f_{i+1}$, with $f_i > f_{i+1}$ (the case $f_i=f_{i+1}$ is studied in \ref{sec:f_iequality}), and assume equal variances\footnote{
    Assuming equal variances simplifies the exposition without altering the substance of the argument. In the general case, different variances would lead to additional intersection points and require a case-by-case analysis. However, the core intuition concerning the separability between the two distributions remains unchanged.
} $\sigma_i^2 = \sigma_{i+1}^2$. As a first step, we observe that if the variances coincide, the two Gaussian distributions intersect at exactly one point. Let $p(x)=\mathcal{N}(f_i,\sigma_i^2)$ and $q(x)=\mathcal{N}(f_{i+1},\sigma_{i+1}^2)$. Solving $p(x)=q(x)$ is equivalent to requiring $\log(p(x))-\log(q(x))=0$. From the expression of the Gaussian density, this leads to 
\[
\log(p(x))= \log\Bigl(\frac{1}{\sqrt{2\pi\sigma_i^2}}e^{-\frac{(x-f_i)^2}{2\sigma_i^2}}\Bigr) = -\frac{1}{2}\log(2\pi)-\frac{1}{2}\log(\sigma_i^2)-\frac{1}{2}\frac{(x-f_i)^2}{\sigma_i^2},
\]
then
\begin{equation}  \label{eq:intersection}
\log(p(x))-\log(q(x))= \frac{1}{2}\Bigl(-\log(\sigma_i^2)+\log(\sigma_{i+1}^2)-\frac{(x-f_i)^2}{\sigma_i^2}+\frac{(x-f_{i+1})^2}{\sigma_{i+1}^2}\Bigr).
\end{equation}
Since $\sigma_i^2=\sigma_{i+1}^2$, \eqref{eq:intersection} reduces to
\begin{equation*}
\begin{aligned}
-\frac{1}{\sigma_i^2}(x^2+f_i^2-2xf_i)+
\frac{1}{\sigma_i^2}(x^2+f_{i+1}^2-2xf_{i+1})&=0 \\
\frac{1}{\sigma_i^2}(2xf_i-f_i^2-2xf_{i+1}+f_{i+1}^2)&=0 \\
x&=\frac{f_i+f_{i+1}}{2}.
\end{aligned}
\end{equation*}
Returning to our main discussion, we now consider the distance $d$ \eqref{eq:distance} applied to $p$ and $q$. Since the support of a Gaussian distribution is $\mathbb{R}$, the integrals extend over the whole real line (we omit this notation for brevity). In this case, the metric reduces to 
\[
d(p,q)=1-\frac{\int \min(p(x),q(x))}{\int \max(p(x),q(x))}.
\]
In this setting there exists only one intersection $\tilde{x}=(f_i+f_{i+1})/2$, and since $f_i>f_{i+1}$
\begin{equation} \label{eq:metric_rewritten}
\begin{aligned}
d(p,q) = 1- \frac{\int_{-\infty}^{\tilde{x} }p(x)+\int_{\tilde{x}}^{\infty}q(x)}{\int_{-\infty}^{\tilde{x}} q(x)+\int_{\tilde{x}}^{\infty} p(x)}&=1-\frac{(1-\int_{\tilde{x}}^{\infty}p(x))+(1-\int_{-\infty}^{\tilde{x}} q(x))}{\int_{-\infty}^{\tilde{x}} q(x)+\int_{\tilde{x}}^{\infty}p(x)} 
\end{aligned} 
\end{equation}
Now, let us set 
\[
\Delta:= f_i-f_{i+1}, \qquad z:= \frac{\Delta}{2\sigma_i}=\frac{f_i-f_{i+1}}{2\sigma_i}.
\]
Using the Gaussian cumulative distribution function $\Phi$, the relevant integrals are 
\begin{align*}
\int_{-\infty}^{\tilde{x}} p(x) &= \Phi\!\bigl((\tilde{x}-f_i)/\sigma\bigr) = \Phi(-z)=1-\Phi(z),\\
\int_{\tilde{x}}^{\infty} p(x) &= 1-\Phi(-z)=\Phi(z)
\end{align*}
To compute the integrals concerning $q(x)$, we can observe that 
\[
q(x) = \mathcal{N}(f_{i+1},\sigma_i^2) = N(f_i-\Delta,\sigma_i^2) = p(x+\Delta).
\]
As a consequence,
\[
\int_{-\infty}^{\tilde{x}} q(x)\,dx = \int_{-\infty}^{\tilde{x}} p(x+\Delta)\,dx = \int_{-\infty}^{\tilde{x}+\Delta}p(t)\,dt = \Phi \Bigl(\frac{\tilde{x}+\Delta-f_i}{\sigma_i}\Bigr)=\Phi(z)
\]
since
\[
\tilde{x}+\Delta-f_i = \frac{f_i+f_{i+1}}{2}+(f_i-f_{i+1})-f_i = \frac{f_i-f_{i+1}}{2}= \frac{\Delta}{2}.
\]
Now, $\int_{\tilde{x}}^{\infty} q(x)= 1-\Phi(z)$ and so
\eqref{eq:metric_rewritten} becomes:
\[
d(p,q) = 1-\frac{2(1-\Phi(z))}{2\Phi(z)}=\frac{2\Phi(z)-1}{\Phi(z)}.
\]
As immediate observations we have:
\begin{itemize}
    \item For $\Delta \to 0^+$ (namely $z\to 0^+$), we have $\Phi(z)\to 1/2$ and so
    \[
    d(p,q) \to \frac{2\cdot \frac{1}{2}-1}{\frac{1}{2}}=0.
    \]
    \item For $\Delta \to \infty$ (namely $z\to \infty$), we have $\Phi(z)\to 1$ and so
    \[
    d(p,q) \to \frac{2\cdot 1-1}{\frac{1}{2}}=1.
    \]
\end{itemize}
In conclusion, we have formally shown in a simple setting how the separation affects the metric $d$ and so the stability of the ranking of the AM.

\subsection{Experiments on synthetic datasets}\label{sec:sythetic_empirical_validation}
To validate the theoretical analysis discussed above, we designed a synthetic experiment. The dataset consists of $n=10$ features with ground-truth coefficients evenly spaced in the interval $[r_{min}, 1]$, where $r_{min}$ is a parameter (indicating the minimum of the range) controlling the closeness of the features:
\[
\mathbf{f} = (f_1, \dots, f_n), \quad f_i = r_{min} + (1-r_{min}) \frac{n-i}{n-1}, \quad i = 1, \dots, n.
\]

The input values $X \in \mathbb{R}^{N \times n}$ for each feature are sampled independently from a standard normal distribution:
\[
X_{ij} \sim \mathcal{N}(0,1), \quad i=1,\dots,N, \ j=1,\dots,n,
\]
where $N=2000$ is the total number of samples. The continuous target $y \in \mathbb{R}^N$ is generated as a linear combination of the features with added Gaussian noise:
\[
y_i = \sum_{j=1}^{n} f_j X_{ij} + \epsilon_i, \quad \epsilon_i \sim \mathcal{N}(0,\sigma^2), \quad i=1,\dots,N,
\]
where we set $\sigma=0.5$. The binary target is then obtained by thresholding $y$ at 0: $y_i^{\text{bin}} = \mathbf{1}\{y_i > 0\}$ for $i=1,\ldots,N$.

For the analysis, a RF classifier was trained on the synthetic datasets, and LIME explanations were computed 50 times for each of the first 50 test instances. The results are reported in \autoref{fig:synthetic_k_stability}. We consider two settings with 10 and 5 features, respectively, while progressively reducing the separation between the ground-truth feature importances by sampling the coefficients from the intervals $[0.4,1]$, $[0.6,1]$, and $[0.8,1]$.

In both cases, k-stability decreases as the feature importances become closer, confirming that the proposed metric is sensitive to attribution separability. Moreover, the 5-feature setting generally yields higher stability values than the 10-feature setting, suggesting that increasing the number of features further reduces the separability of the induced rankings. Overall, these controlled experiments support the theoretical analysis and validate the ability of k-stability to capture changes in feature importance separation.

\begin{figure}[!h]
    \centering

    \begin{subfigure}{\columnwidth}
        \centering
        \includegraphics[width=\columnwidth]{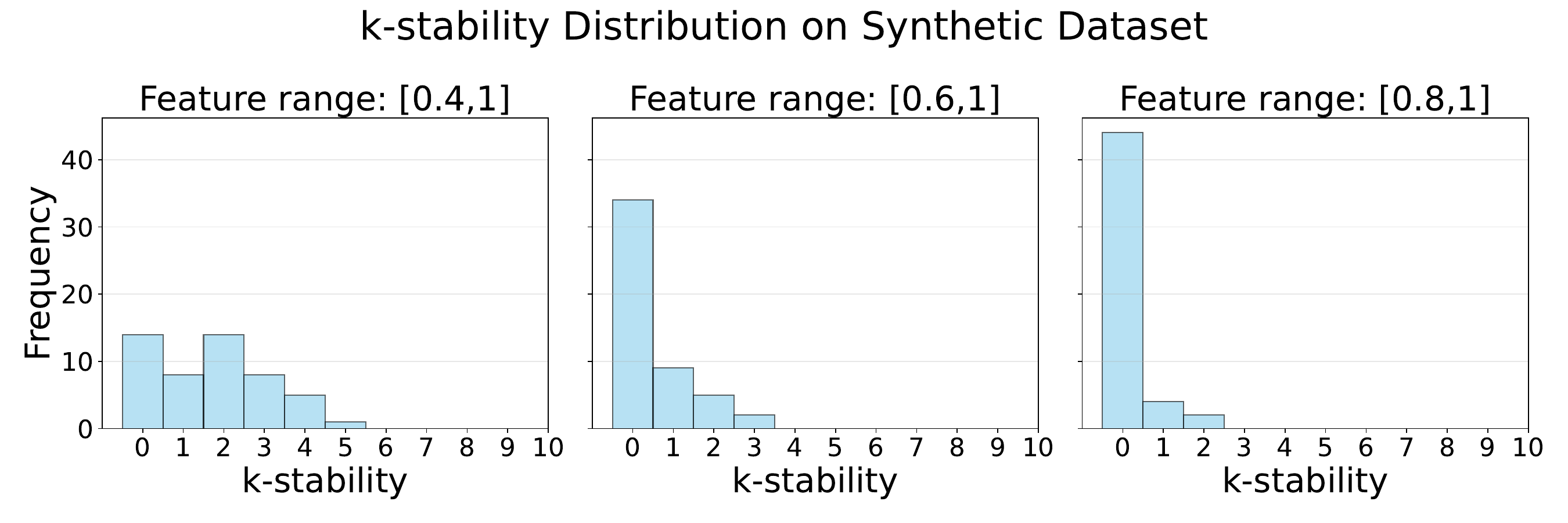}
        \caption{Synthetic dataset with 10 features.}
        \label{fig:synthetic_k_stability_10}
    \end{subfigure}

    \vspace{2mm}

    \begin{subfigure}{\columnwidth}
        \centering
        \includegraphics[width=\columnwidth]{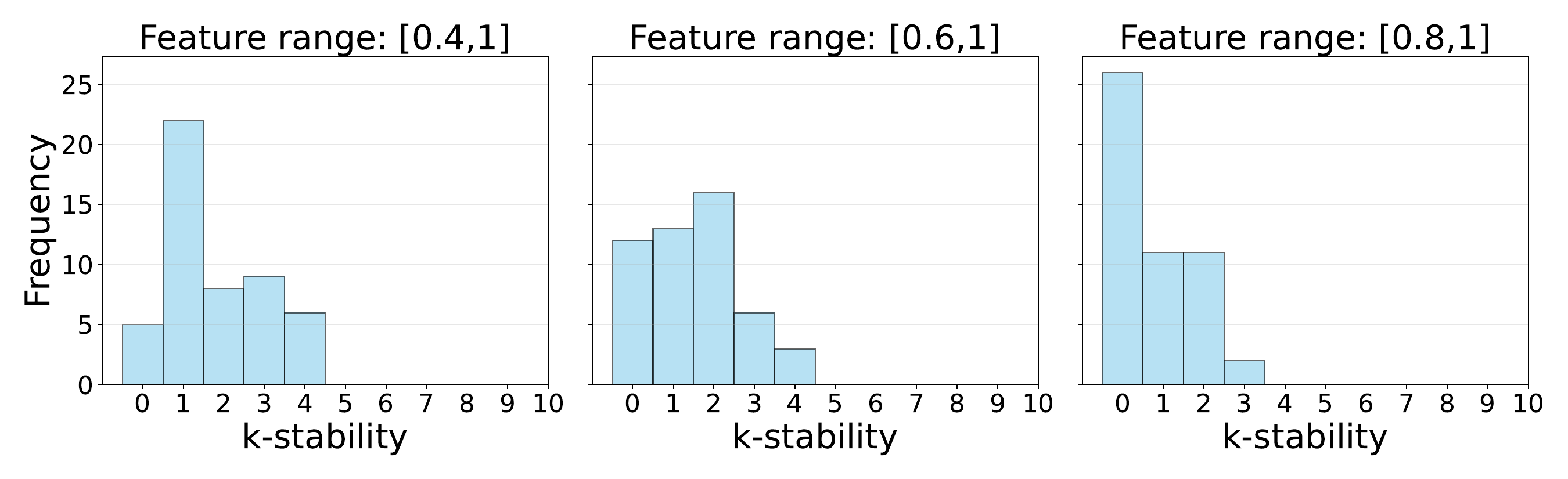}
        \caption{Synthetic dataset with 5 features.}
        \label{fig:synthetic_k_stability_5}
    \end{subfigure}

    \caption{Validation of k-stability on synthetic datasets. Ground-truth feature importances are sampled from progressively narrower intervals, reducing their separation. In both settings, k-stability decreases as feature importances become closer, while the 5-feature dataset exhibits generally higher stability than the 10-feature dataset.}
    \label{fig:synthetic_k_stability}
\end{figure}

\subsection{The case $f_i=f_{i+1}$} \label{sec:f_iequality}
To complement the previous analysis, we now focus our attention to the case of an explainer converging in mean to the same value for two attributions and what insights we can derive from the metric $d$. Let $p(x)=\mathcal{N}(\mu,\sigma_1^2),\, q(x)=\mathcal{N}(\mu,\sigma_2^2)$ with $\sigma_1^2\neq\sigma_2^2$. Following the same strategy, the points of intersection are solutions to 
\[
-\frac{1}{2}\log\sigma_1^2-\frac{(x-\mu)^2}{2\sigma_1^2}
=
-\frac{1}{2}\log\sigma_2^2-\frac{(x-\mu)^2}{2\sigma_2^2},
\]
which can be written as
\begin{equation*}
\begin{aligned}   
(x-\mu)^2\!\left(\frac{1}{\sigma_1^2}-\frac{1}{\sigma_2^2}\right) 
=\log\frac{\sigma_2^2}{\sigma_1^2}\,\,\implies
(x-\mu)^2
=\log\!\left(\frac{\sigma_2^2}{\sigma_1^2}\right)\,\frac{\sigma_1^2\sigma_2^2}{\sigma_2^2-\sigma_1^2}.
\end{aligned}
\end{equation*}
Now, since the right term is always positive, the solutions are:
\[
x=\mu\pm\sqrt{\log\left(\frac{\sigma_2^2}{\sigma_1^2}\right)\frac{\sigma_1^2\sigma_2^2}{\sigma_2^2-\sigma_1^2}}.
\]
We are now ready to compute the metric $d$ in this case. 
Let us assume without loss of generality $\sigma_1<\sigma_2$ (if $\sigma_1=\sigma_2$ then the two distributions are identical and therefore their feature importances indistinguishable). Let us denote the two intersections as
\[
x_{1,2}=\mu\mp R, \qquad R = \sqrt{\log\left(\frac{\sigma_2^2}{\sigma_1^2}\right)\frac{\sigma_1^2\sigma_2^2}{\sigma_2^2-\sigma_1^2}}.
\]
Since, the narrower Gaussian $p$ prevails near $\mu$ and the wider one $q$ prevails at the tails, we have the natural partition
\[
\begin{aligned}
\int\min(p,q)
&=\int_{-\infty}^{x_1} p(x)+\int_{x_1}^{x_2} q(x)+\int_{x_2}^{\infty} p(x),\\
\int\max(p,q)
&=\int_{-\infty}^{x_1} q(x)+\int_{x_1}^{x_2} p(x)+\int_{x_2}^{\infty} q(x).
\end{aligned}
\]
By using the normal cumulative distribution function $\Phi$ and exploiting the symmetry of \(x_{1,2}=\mu\mp R\) w.r.t, $\mu$, we have, letting $r_1=R/\sigma_1$ and $r_2=R/\sigma_2$,
\[
\begin{aligned}
\int_{-\infty}^{x_1} p(x)\,dx &= \Phi(-r_1)=1-\Phi(r_1), &\qquad \int_{-\infty}^{x_1} q(x)\,dx &= \Phi(-r_2)=1-\Phi(r_2);
\\
\int_{x_1}^{x_2} q(x)\,dx &= 2\Phi(r_2)-1, &\qquad \int_{x_1}^{x_2} p(x)\,dx &= \Phi(r_1)-\Phi(-r_1)=2\Phi(r_1)-1; \\
\int_{x_2}^{\infty} p(x)\,dx &= 1-\Phi(r_1), & \qquad \int_{x_2}^{\infty} q(x)\,dx &= 1-\Phi(r_2).
\end{aligned}
\]
Now adding up the terms:
\[
\begin{aligned}
\int\min(p,q) &= \bigl(1-\Phi(r_1)\bigr)+\bigl(2\Phi(r_2)-1\bigr)+\bigl(1-\Phi(r_1)\bigr)
=1+2\bigl(\Phi(r_2)-\Phi(r_1)\bigr),\\[4pt]
\int\max(p,q) &= \bigl(1-\Phi(r_2)\bigr)+\bigl(2\Phi(r_1)-1\bigr)+\bigl(1-\Phi(r_2)\bigr)
=1+2\bigl(\Phi(r_1)-\Phi(r_2)\bigr).
\end{aligned}
\]
As a consequence, setting $A=\Phi(r_1)-\Phi(r_2)$,
\[
d(p,q)=1-\frac{\int\min(p,q)}{\int\max(p,q)}
=1-\frac{1-2A}{1+2A}
=\frac{4A}{1+2A}.
\]
As immediate observations we have:
\begin{itemize}
  \item If $\sigma_1\to\sigma_2$ then $R\to 0$, therefore $r_1,r_2\to 0$ and $A\to0$. As a consequence $d(p,q)\to 0$: the densities are almost equivalent and the stability metric yields small values. 
  \item If $\sigma_2/\sigma_1 \to +\infty$ (or symmetrically  $\sigma_1/\sigma_2 \to 0^+$), then we have
  \[
  r_1\to+\infty,\qquad r_2\to 0,
  \]
  hence \(A=\Phi(r_1)-\Phi(r_2)\to 1-\tfrac12=\tfrac12\) and so
  \[
  d(p,q)\to \frac{4\cdot (1/2)}{1+2\cdot (1/2)}=1.
  \]
The intuition is that a distribution sharply concentrated around its mean reflects a confident feature importance assignment, whereas a distribution with larger variance indicates greater uncertainty. The metric $d$ correctly captures this difference by assigning a large distance between the two distributions. However, when the corresponding features have the same importance, their relative order remains indistinguishable. Therefore, while the metric is well-posed, this example highlights a limitation of $k$-stability.
\end{itemize}

\end{document}